\documentclass[twoside]{article}
\pdfoutput=1
\usepackage[accepted]{icml2016}
\usepackage{times}
\usepackage[utf8]{inputenc}

\usepackage{amsmath}
\usepackage{amssymb}
\usepackage{amsfonts}
\usepackage{amsthm}
\theoremstyle{plain}
\newtheorem{theorem}{Theorem}
\newtheorem{lemma}{Lemma}
\newtheorem{proposition}{Proposition}
\newtheorem{cor}{Corollary}
\theoremstyle{definition}

\newtheorem{remark}{Remark}

\usepackage{natbib}
\usepackage{url}

\newcommand{\RETURN}{\STATE return}

\usepackage{algorithm, algorithmic}
\usepackage{subfigure}

\usepackage{graphicx}
\graphicspath{{./images/}}
\usepackage{shortcuts}
\usepackage{booktabs}

\icmltitlerunning{Gossip Dual Averaging for Decentralized Optimization of Pairwise Functions}

\begin{document}

\twocolumn[
\icmltitle{Gossip Dual Averaging for Decentralized Optimization of \\
  Pairwise Functions}


\icmlauthor{Igor Colin}{igor.colin@telecom-paristech.fr}
\icmladdress{LTCI, CNRS, T\'el\'ecom ParisTech, Unversit\'e Paris-Saclay,
  75013 Paris, France}
\icmlauthor{Aur\'elien Bellet}{aurelien.bellet@inria.fr}
\icmladdress{Magnet Team, INRIA Lille -- Nord Europe,
            59650 Villeneuve d'Ascq, France}
\icmlauthor{Joseph Salmon}{joseph.salmon@telecom-paristech.fr}
\icmlauthor{St\'ephan Cl\'emen\c{c}on}{stephan.clemencon@telecom-paristech.fr}
\icmladdress{LTCI, CNRS, T\'el\'ecom ParisTech, Unversit\'e Paris-Saclay,
  75013 Paris, France}

\icmlkeywords{decentralized optimization, dual averaging, gossip protocols}
\vskip 0.3in
]


\begin{abstract}

In decentralized networks (of sensors, connected objects, \textit{etc.}), there is an important need for efficient algorithms to optimize a global cost function, for instance to learn a global model from the local data collected by each computing unit.
In this paper, we address the problem of decentralized minimization of pairwise functions of the data points, where these points are distributed over the nodes of a graph defining the communication topology of the network.
This general problem finds applications in ranking, distance metric learning and graph inference, among others.
We propose new gossip algorithms based on dual averaging which aims at solving such problems both in synchronous and asynchronous settings. The proposed framework is flexible enough to deal with constrained and regularized variants of the optimization problem. Our theoretical analysis reveals that the proposed algorithms preserve the convergence rate of centralized dual averaging up to an additive bias term. We present numerical simulations on Area Under the ROC Curve (AUC) maximization and metric learning problems which illustrate the practical interest of our approach.

\end{abstract}


\section{Introduction}
\label{sec:introduction}

The increasing popularity of large-scale and fully decentralized computational architectures, fueled for instance by the advent of the ``Internet of Things'', motivates the development of efficient optimization algorithms adapted to this setting.
An important application is machine learning in wired and wireless networks of agents (sensors, connected objects, mobile phones, \textit{etc.}), where the agents seek to minimize a global learning objective which depends of the data collected locally by each agent. In such networks, it is typically impossible to efficiently centralize data or to globally aggregate intermediate results: agents can only communicate with their immediate neighbors (\eg agents within a small distance), often in a completely asynchronous fashion. Standard distributed optimization and machine learning algorithms (implemented for instance using MapReduce/Spark) require a coordinator node and/or to maintain synchrony, and are thus unsuitable for use in decentralized networks.

In contrast, \emph{gossip algorithms} \citep{Tsitsiklis1984a,Boyd2006a,Kempe2003a,Shah2009a} are tailored to this setting because they only rely on simple peer-to-peer communication: each agent only exchanges information with one neighbor at a time. Various gossip algorithms have been proposed to solve the flagship problem of decentralized optimization, namely to find a parameter vector $\theta$ which minimizes an average of convex functions $(1/n)\sum_{i=1}^n f(\theta; x_i)$, where the data $x_i$ is only known to agent $i$. The most popular algorithms are based on (sub)gradient descent \citep{Johansson2010a,Nedic2009a,Ram2010a,Bianchi2013a}, ADMM \citep{Wei2012a,Wei2013a,Iutzeler2013a} or dual averaging \citep{Duchi2012a,yuan2012distributed,lee2015decentralized,Tsianos2015a}, some of which can also accommodate constraints or regularization on $\theta$. The main idea underlying these methods is that each agent seeks to minimize its local function by applying local updates (\eg gradient steps) while exchanging information with neighbors to ensure a global convergence to the consensus value.

In this paper, we tackle the problem of minimizing an average of \emph{pairwise} functions of the agents' data:
\begin{equation}\label{eq:sumfij}
\min_{\theta} \frac{1}{n^2}\sum_{1 \leq i, j \leq n} f(\theta; x_i, x_j).
\end{equation}
This problem finds numerous applications in statistics and machine learning, \eg Area Under the ROC Curve (AUC) maximization \citep{Zhao2011a}, distance/similarity learning \citep{Bellet2015c}, ranking \citep{Clemencon2008a}, supervised graph inference \citep{Biau2006a} and multiple kernel learning \citep{Kumar2012a}, to name a few.
As a motivating example, consider a mobile phone application which locally collects information about its users. The provider could be interested in learning pairwise similarity functions between users in order to group them into clusters or to recommend them content without having to centralize data on a server (which would be costly for the users' bandwidth) or to synchronize phones.

The main difficulty in Problem \eqref{eq:sumfij} comes from the fact that each term of the sum depends on two agents $i$ and $j$, making the local update schemes of previous approaches impossible to apply unless data is exchanged between nodes. Although gossip algorithms have recently been introduced to evaluate such pairwise functions for a \emph{fixed} $\theta$ \citep{Pelckmans2009a,Colin_Bellet_Salmon_Clemencon15}, to the best of our knowledge, efficiently finding the \emph{optimal solution} $\theta$ in a decentralized way remains an open challenge.
Our contributions towards this objective are as follows. We propose new gossip algorithms based on dual averaging \citep{Nesterov2009a,Xiao10} to efficiently solve Problem~\eqref{eq:sumfij} and its constrained or regularized variants. Central to our methods is a light data propagation scheme which allows the nodes to compute \emph{biased} estimates of the gradients of functions in \eqref{eq:sumfij}.
We then propose a theoretical analysis of our algorithms both in synchronous and asynchronous settings establishing their convergence under an additional hypothesis that the bias term decreases fast enough over the iterations (and we have observed such a fast decrease in all our experiments).
Finally, we present some numerical simulations on Area Under the ROC Curve (AUC) maximization and metric learning problems. These experiments illustrate the practical performance of the proposed algorithms and the influence of network topology, and show that in practice the influence of the bias term is negligible as it decreases very fast with the number of iterations.

The paper is organized as follows. Section~\ref{sec:statement_notation} formally introduces the problem of interest and briefly reviews the dual averaging method, which is at the root of our approach.
Section~\ref{sec:distr-dual-aver} presents the proposed gossip algorithms and their convergence analysis. Section~\ref{sec:experiments} displays our numerical simulations. Finally, concluding remarks are collected in Section~\ref{sec:conclusion}.













\section{Preliminaries}
\label{sec:statement_notation}

\subsection{Definitions and Notation}
\label{subsec:notation}

For any integer $p > 0$,  we denote by $[p]$ the set $\{ 1, \ldots, p \}$ and by $\vert F \vert$ the cardinality of any finite set $F$. We denote an undirected graph by $\mathcal{G} = (V, E)$, where $V=[n]$ is the set of vertices and $E\subseteq V\times V$ is the set of edges.
A node $i \in V$ has degree $d_i = | \{j : (i, j) \in E \} |$. $\mathcal{G}$ is connected if for all $(i, j) \in V^2$ there exists a path connecting $i$ and $j$; it is bipartite if there exist $S, T \subset V$ such that $S \cup T = V$, $S \cap T = \emptyset$ and $E \subseteq (S \times T) \cup (T \times S)$. The graph Laplacian of $\mathcal{G}$
is denoted by $L(\mathcal{G}) = D(\mathcal{G}) - A(\mathcal{G})$, where $D(\mathcal{G})$ and $A(\mathcal{G})$ are respectively the degree and the adjacency matrices of $\mathcal{G}$.

The transpose of a matrix $M \in \bbR^{n \times n}$ is denoted by $M^{\top}$. A matrix $P \in \bbR^{n \times n}$ is termed stochastic whenever $P \geq 0$ and $P \1_n = \1_n$, where $\1_n = (1, \ldots, 1)^{\top}\in \bbR^n$, and bi-stochastic whenever both $P$ and $P^{\top}$ are stochastic. We denote by $I_n$ the identity matrix in $\bbR^{n\times n}$, by $(e_1,\dots,e_n)$ the canonical basis of $\bbR^n$, by $\mathbb{I}_{\{\mathcal{E}\}}$ the indicator function of any event $\mathcal{E}$ and by $\|\cdot\|$ the usual $\ell_2$-norm.
For $\theta \in \bbR^d$ and $g : \bbR^d \rightarrow \bbR$, we denote by $\nabla g(\theta)$ the gradient of $g$ at $\theta$.
Finally, given a collection of vectors $u_1,\dots,u_n$, we denote by $\bar{u}^n=(1/n)\sum_{i=1}^nu_i$ its empirical mean.

\subsection{Problem Statement}
\label{subsec:problem_statement}

We represent a network of $n$ agents as an undirected graph $\mathcal{G} = ([n], E)$, where each node $i\in[n]$ corresponds to an agent and $(i,j)\in E$ if nodes $i$ and $j$ can exchange information directly (\textit{i.e.}, they are neighbors). For ease of exposition, we assume that each node $i \in [n]$ holds a single data point $x_i\in\mathcal{X}$.  Though restrictive in practice, this assumption can easily be relaxed, but it would lead to more technical details to handle the storage size, without changing the overall analysis (see supplementary material for details).

Given $d > 0$, let $f : \bbR^d \times \featspace \times \featspace \to \bbR$ a differentiable and convex function with respect to the first variable. We assume that for any $(x, x') \in \featspace^2$, there exists $L_{f} > 0$ such that $f(\cdot; x, x')$ is $L_f$-Lipschitz (with respect to the $\ell_2$-norm).
Let $\psi:\bbR^d \to \bbR^+$ be a non-negative, convex, possibly non-smooth, function such that, for simplicity, $\psi(0)=0$.
We aim at solving the following optimization problem:
\begin{equation}
  \label{eq:emp_opt_problem}
  \min_{\theta \in \bbR^d} \frac{1}{n^2} \sum_{1 \leq i, j \leq n} f(\theta; x_i, x_j)  + \psi(\theta).
\end{equation}
In a typical machine learning scenario, Problem~\eqref{eq:emp_opt_problem} is a (regularized) empirical risk minimization problem and $\theta$ corresponds to the model parameters to be learned. The quantity $f(\theta; x_i, x_j)$ is a pairwise loss measuring the performance of the model $\theta$ on the data pair $(x_i, x_j)$, while $\psi(\theta)$ represents a regularization term penalizing the complexity of $\theta$. Common examples of regularization terms include indicator functions of a closed convex set to model explicit convex constraints, or norms enforcing specific properties such as sparsity (a canonical example being the $\ell_1$-norm).

Many machine learning problems can be cast as Problem \eqref{eq:emp_opt_problem}. For instance, in AUC maximization \citep{Zhao2011a}, binary labels $(\ell_1,\ldots,\ell_n)\in\{-1,1\}^n$ are assigned to the data points and we want to learn a (linear) scoring rule $x\mapsto x^\top\theta$ which hopefully gives larger scores to positive data points than to negative ones.
One may use the logistic loss
$$f(\theta; x_i, x_j)=\mathbb{I}_{\{\ell_i>\ell_j\}} \log\left(1+\exp((x_j-x_i)^\top\theta)\right),$$
and the regularization term $\psi(\theta)$ can be the square $\ell_2$-norm of $\theta$ (or the $\ell_1$-norm when a sparse model is desired). Other popular instances of Problem~\eqref{eq:emp_opt_problem} include metric learning \citep{Bellet2015c}, ranking \citep{Clemencon2008a}, supervised graph inference \citep{Biau2006a} and multiple kernel learning \citep{Kumar2012a}.

For notational convenience, we denote by $f_i$ the partial function $(1/n)\sum_{j = 1}^n f(\cdot; x_i, x_j)$ for $i\in[n]$ and by $\avf = (1/n) \sum_{i = 1}^n f_i$. Problem~\eqref{eq:emp_opt_problem} can then be recast as:
\begin{equation}
  \label{eq:alt_emp_opt_problem}
  \min_{\theta \in \bbR^d} R_n(\theta) = \avf(\theta)+ \psi(\theta).
\end{equation}
Note that the function $\avf$ is $L_{f}$-Lipschitz, since all the $f_i$ are $L_{f}$-Lipschitz.

\begin{remark}
  Throughout the paper we assume that the function $f$ is differentiable, but we expect all our results to hold even when $f$ is non-smooth, for instance in $L_1$-regression problems or when using the hinge loss. In this case, one simply needs to replace gradients by subgradients in our algorithms, and a similar analysis could be performed.
\end{remark}

\subsection{Centralized Dual Averaging}
\label{subsec:centralize_dual_averaging}

\begin{algorithm}[t]
  \small
  \caption{Stochastic dual averaging in the centralized setting}
  \label{alg:dual_averaging_sto}
  \begin{algorithmic}[1]
    \REQUIRE Step size $(\gamma(t))_{t \geq 0} > 0$.
    \STATE Initialization: $\theta = 0$, $\avtheta = 0$, $z = 0$.
    \FOR{$t = 1, \ldots, T$}
    \STATE Update $z \gets z + g(t)$, where $\bbE[g(t)|\theta] = \nabla \avf(\theta)$
    \STATE Update $\theta \gets \smoothop_t(z)$
    \STATE Update $\avtheta \gets \left( 1 - \frac{1}{t} \right) \avtheta + \frac{1}{t} \theta$
    \ENDFOR
    \STATE \textbf{return} $\avtheta$
  \end{algorithmic}
\end{algorithm}

In this section, we review the stochastic dual averaging optimization algorithm \citep{Nesterov2009a,Xiao10} to solve Problem \eqref{eq:emp_opt_problem} in the centralized setting (where all data lie on the same machine). This method is at the root of our gossip algorithms, for reasons that will be made clear in Section~\ref{sec:distr-dual-aver}.
To explain the main idea behind dual averaging, let us first consider the iterations of Stochastic Gradient Descent (SGD), assuming $\psi \equiv 0$ for simplicity:
$$\theta(t+1) = \theta(t) - \gamma(t)g(t),$$
where $\bbE[g(t)|\theta(t)] = \nabla\avf(\theta(t))$, and $(\gamma(t))_{t \geq 0}$ is a non-negative non-increasing step size sequence. For SGD to converge to an optimal solution, the step size sequence must satisfy $\gamma(t)\underset{t \to +\infty}{\longrightarrow}0$ and $\sum_{t=0}^\infty\gamma(t) = \infty$. As noticed by \citet{Nesterov2009a}, an undesirable consequence is that new gradient estimates are given smaller weights than old ones. Dual averaging aims at integrating all gradient estimates with the same weight.

Let $(\gamma(t))_{t \geq 0}$ be a positive and non-increasing step size sequence. The dual averaging algorithm maintains a sequence of iterates $(\theta(t))_{t > 0}$, and a sequence $(z(t))_{t \geq 0}$ of ``dual'' variables which collects the sum of the unbiased gradient estimates seen up to time $t$. We initialize to $\theta(1)=z(0)=0$. At each step $t>0$, we compute an unbiased estimate $g(t)$ of $\nabla\avf(\theta(t))$. The most common choice is to take $g(t)=\nabla f(\theta; x_{i_t}, x_{j_t})$ where $i_t$ and $j_t$ are drawn uniformly at random from $[n]$. We then set $z(t + 1)=z(t)+g(t)$ and generate the next iterate with the following rule:
$$ \displaystyle\begin{cases}
\theta(t+1)=\smoothop_t^\psi(z(t + 1)),\\
\smoothop_t^\psi(z):= \displaystyle\argmin_{\theta \in \bbR^d}
    \left\{
    -z^{\top}\theta + \frac{\| \theta \|^2}{2 \gamma(t)}  + t \psi(\theta) \right\}.
\end{cases}
$$
When it is clear from the context, we will drop the dependence in $\psi$ and simply write $\smoothop_t(z)=\smoothop_t^\psi(z)$.
\begin{remark}
  Note that $\smoothop_t(\cdot)$ is related to the proximal operator of a function $\phi:\bbR^d \to \bbR$ defined by $\prox_{\phi}(x)=\argmin_{z\in \bbR^d} \left(\|z-x\|^2/2 +\phi(x)\right)$. Indeed, one can write:
$$\smoothop_t(z)=\prox_{t\gamma(t) \psi}\left(\gamma(t)z\right).$$
For many functions $\psi$ of practical interest, $\smoothop_t(\cdot)$ has a closed form solution.
For instance, when $\psi=\|\cdot\|^2$, $\smoothop_t(\cdot)$ corresponds to a simple scaling, and when $\psi=\|\cdot\|_1$ it is a soft-thresholding operator. If $\psi$ is the indicator function of a closed convex set $\mathcal{C}$, then $\smoothop_t(\cdot)$ is the projection operator onto $\mathcal{C}$.
\end{remark}

The dual averaging method is summarized in Algorithm~\ref{alg:dual_averaging_sto}. If $\gamma(t)\propto 1/\sqrt{t}$ then for any $T>0$:
$$\bbE_T  \big[R_n(\avtheta(T)) - R_n(\theta^*) \big] = \mathcal{O}(1/\sqrt{T}),$$
where $\theta^*\in\argmin_{\theta\in\bbR^d} R_n(\theta)$, $\avtheta(T)=\frac{1}{T}\sum_{i=1}^T\theta(t)$ is the averaged iterate and $\bbE_T$ is the expectation over all possible sequences $(g(t))_{1 \leq t \leq T}$. A precise statement of this result along with a proof can be found in the supplementary material for completeness.

Notice that dual averaging cannot be easily adapted to our decentralized setting. Indeed, a node cannot compute an unbiased estimate of its gradient: this would imply an access to the entire set of data points, which violates the communication and storage constraints. Therefore, data points have to be appropriately propagated during the optimization procedure, as detailed in the following section.





\section{Pairwise Gossip Dual Averaging}
\label{sec:distr-dual-aver}


We now turn to our main goal, namely to develop efficient gossip algorithms for solving Problem~\eqref{eq:emp_opt_problem} in the decentralized setting.
The methods we propose rely on dual averaging (see Section~\ref{subsec:centralize_dual_averaging}). This choice is guided by the fact that the structure of the updates makes dual averaging much easier to analyze in the distributed setting than sub-gradient descent when the problem is constrained or regularized. This is because dual averaging maintains a simple sum of sub-gradients, while the (non-linear) smoothing operator $\smoothop_t$ is applied separately.

Our work builds upon the analysis of \citet{Duchi2012a}, who proposed a distributed dual averaging algorithm to optimize an average of \emph{univariate} functions $f(\cdot; x_i)$. In their algorithm, each node $i$ computes \emph{unbiased} estimates of its local function $\nabla f(\cdot; x_i)$ that are iteratively averaged over the network.
Unfortunately, in our setting, the node $i$ cannot compute unbiased estimates of $\nabla f_i(\cdot) = \nabla (1/n)\sum_{j = 1}^n f(\cdot; x_i, x_j)$: the latter depends on all data points while each node $i \in [n]$ only holds $x_i$. To go around this problem, we rely on a gossip data propagation step \citep{Pelckmans2009a,Colin_Bellet_Salmon_Clemencon15} so that the nodes are able to compute \emph{biased} estimates of $\nabla f_i(\cdot)$ while keeping the communication and memory overhead to a small level for each node.



We present and analyze our algorithm in the synchronous setting in Section~\ref{subsec:synchronous-setting}. We then turn to the more intricate analysis of the asynchronous setting in Section~\ref{subsec:asynchronous-setting}.

\subsection{Synchronous Setting}
\label{subsec:synchronous-setting}

\begin{algorithm}[t]
  \small
  \caption{Gossip dual averaging for pairwise function in synchronous setting}
  \label{alg:gossip_dual_averaging_sync}
  \begin{algorithmic}[1]
    \REQUIRE Step size $(\gamma(t))_{t \geq 1} > 0$.
    \STATE Each node $i$ initializes $y_i = x_i$, $z_i = \theta_i = \avtheta_i = 0$.
    \FOR{$t = 1, \ldots, T$}
    \STATE Draw $(i, j)$ uniformly at random from $E$
    \STATE Set $z_i, z_j \gets \frac{z_i + z_j}{2}$
    \STATE Swap auxiliary observations: $y_i \leftrightarrow y_j$
    \FOR{$k = 1, \ldots, n$}
    \STATE Update $z_k \gets z_k + \nabla_{\theta} f(\theta_k; x_k, y_k)$
    \STATE Compute $\theta_k \gets \smoothop_t(z_k)$
    \STATE Average $\avtheta_k \gets \left( 1 - \frac{1}{t} \right) \avtheta_k + \frac{1}{t} \theta_k$
    \ENDFOR
    \ENDFOR
    \STATE \textbf{return} Each node $k$ has $\avtheta_k$
  \end{algorithmic}
\end{algorithm}

In the synchronous setting, we assume that each node has access to a global clock such that every node can update simultaneously at each tick of the clock. Although not very realistic, this setting allows for simpler analysis. We assume that the scaling sequence $(\gamma(t))_{t \geq 0}$ is the same for every node.
At any time, each node $i$ has the following quantities in its local memory register: a variable $z_i$ (the gradient accumulator), its original observation $x_i$, and an \textit{auxiliary observation} $y_i$, which is initialized at $x_i$ but will change throughout the algorithm as a result of data propagation.

The algorithm goes as follows. At each iteration, an edge $(i, j)\in E$ of the graph is drawn uniformly at random. Then, nodes $i$ and $j$ average their gradient accumulators $z_i$ and $z_j$, and swap their auxiliary observations $y_i$ and $y_j$. Finally, every node of the network performs a dual averaging step, using their original observation and their current auxiliary one to estimate the partial gradient. The procedure is detailed in Algorithm~\ref{alg:gossip_dual_averaging_sync}, and the following proposition adapts the convergence rate of centralized dual averaging under the hypothesis that the contribution of the bias term decreases fast enough over the iterations.

\begin{theorem}
  \label{thm:gossip_dual_averaging_general_rate}
  Let $\mathcal{G}$ be a connected and non-bipartite graph with $n$ nodes, and let $\theta^*\in\argmin_{\theta\in\bbR^d} R_n(\theta)$. Let $(\gamma(t))_{t \geq 1}$ be a non-increasing and non-negative sequence. For any $i \in [n]$ and any $t\geq 0$, let $z_i(t) \in \bbR^d$ and $\avtheta_i(t) \in \bbR^d$ be generated according to Algorithm~\ref{alg:gossip_dual_averaging_sync}. Then for
  any $i \in [n]$ and $T > 1$, we have:
  \[
    \bbE_T[R_n(\avtheta_i) - R_n(\theta^*)] \leq C_1(T) + C_2(T) + C_3(T),
  \]
  where
  \[
    \left\{
    \begin{aligned}
      C_1(T) &= \frac{1}{2 T \gamma(T)} \| \theta^* \|^2 + \frac{L_f^2}{2T} \sum_{t = 1}^{T - 1} \gamma(t), \\
      C_2(T) &= \frac{3L_f^2}{T\left(1 - \sqrt{\lambda_2^{\mathcal{G}}}\right)}\sum_{t = 1}^{T - 1}\gamma(t),\\
      C_3(T) &= \frac{1}{T} \sum_{t = 1}^{T - 1} \bbE_t[(\omega(t) - \theta^*)^{\top} \bar{\epsilon}^n(t)],
    \end{aligned}
    \right.
  \]
  and $\lambda_2^{\mathcal{G}} < 1$ is the second largest eigenvalue of the matrix $W(\mathcal{G}) = I_n - \frac{1}{|E|} L(\mathcal{G})$.
\end{theorem}
\begin{proof}[Sketch of proof]
First notice that at a given (outer) iteration $t + 1$, $\avz$ is updated as follows:
\begin{equation}
  \label{eq:update_av_z}
  \avz(t + 1) = \avz(t) + \frac{1}{n} \sum_{k = 1}^n d_k(t),
\end{equation}
where $d_k(t)=\nabla_{\theta} f(\theta_k(t); x_k, y_k(t + 1))$ is a biased estimate of $\nabla f_k(\theta_k(t))$. Let $\epsilon_k(t)=d_k(t)-g_k(t)$ be the bias, so that we have $\bbE[g_k(t) | \theta_k(t)] = \nabla f_k(\theta_k(t))$.

Let us define $\projavz(t) = \smoothop_t(\avz(t))$. Using convexity of $R_n$, the gradient's definition and the fact that the functions $\avf$ and $\smoothop_t$ are both $L_f$-Lipschitz, we obtain: for $T \geq 2$ and $i \in [n]$,
  \begin{eqnarray}
      \lefteqn{\bbE_T[ R_n(\avtheta_i(T)) - R_n(\theta^*)]}\nonumber\\
      &\leq& \frac{L_f}{nT} \sum_{t = 2}^T \gamma(t - 1) \sum_{j = 1}^n \bbE_t\Big[ \|z_i(t) - z_j(t) \| \Big]\label{eq:part1a} \\
      &+& \frac{L_f}{nT} \sum_{t = 2}^T \gamma(t - 1) \sum_{j = 1}^n \bbE_t\Big[\| \avz(t) - z_j(t) \| \Big]\label{eq:part1b}\\
      & +& \frac{1}{T} \sum_{t = 2}^T \bbE_t[(\omega(t) - \theta^*)^{\top} \avg(t) ].\label{eq:part2}
  \end{eqnarray}
Using Lemma~4 (see supplementary material), the terms \eqref{eq:part1a}-\eqref{eq:part1b} can be bounded by $C_2(T)$. The term \eqref{eq:part2} requires a specific analysis because the updates are performed using biased estimates. We decompose it as follows:
  \begin{eqnarray}
      \lefteqn{\frac{1}{T} \sum_{t = 2}^T \bbE_t\Big[\omega(t) - \theta^*)^{\top} \bar{g}^n(t) \Big]}\nonumber\\
      &=& \frac{1}{T} \sum_{t = 2}^T \bbE_t\Big[(\omega(t) - \theta^*)^{\top} (\bar{d}^n(t) - \bar{\epsilon}^n(t))\Big] \nonumber\\
      &\leq& \frac{1}{T} \sum_{t = 2}^T \bbE_t\Big[(\omega(t) - \theta^*)^{\top} \bar{d}^n(t)\Big]\label{eq:xiao}\\
      &+& \frac{1}{T} \sum_{t = 2}^{T} \bbE_t\Big[(\omega(t) - \theta^*)^{\top} \bar{\epsilon}^n(t)\Big]\nonumber
      .
  \end{eqnarray}
The term \eqref{eq:xiao} can be bounded by $C_1(T)$ \citep[see][Lemma 9]{Xiao10}.
We refer the reader to the supplementary material for the detailed proof.
\end{proof}

The rate of convergence in Proposition~\ref{thm:gossip_dual_averaging_general_rate} is divided into three parts: $C_1(T)$ is a \emph{data dependent} term which corresponds to the rate of convergence of the centralized dual averaging, while $C_2(T)$ and $C_3(T)$ are \emph{network dependent} terms since $1 - \lambda_2^{\mathcal{G}}=\beta_{n-1}^{\mathcal{G}}/|E|$, where $\beta_{n-1}^{\mathcal{G}}$ is the second smallest eigenvalue of the graph Laplacian $L(\mathcal{G})$, also known as the spectral gap of $\mathcal{G}$. The convergence rate of our algorithm thus improves when the spectral gap is large, which is typically the case for well-connected graphs \citep{chung1997spectral}. Note that $C_2(T)$ corresponds to the network dependence for the distributed dual averaging algorithm of \citet{Duchi2012a} while the term $C_3(T)$ comes from the bias of our partial gradient estimates. In practice, $C_3(T)$ vanishes quickly and has a small impact on the rate of convergence, as shown in Section~\ref{sec:experiments}.






\subsection{Asynchronous Setting}
\label{subsec:asynchronous-setting}

\begin{algorithm}[t]
  \small
  \caption{Gossip dual averaging for pairwise function in asynchronous setting }
  \label{alg:gossip_dual_averaging_async}
  \begin{algorithmic}[1]
    \REQUIRE Step size $(\gamma(t))_{t \geq 0} > 0$, probabilities $(p_k)_{k \in [n]}$.
    \STATE Each node $i$ initializes $y_i = x_i$, $z_i = \theta_i = \avtheta_i = 0$, $m_i = 0$.
    \FOR{$t = 1, \ldots, T$}
    \STATE Draw $(i, j)$ uniformly at random from $E$
    \STATE Swap auxiliary observations: $y_i \leftrightarrow y_j$
    \FOR{$k \in \{i, j\}$}
    \STATE Set $z_k \gets \frac{z_i + z_j}{2}$
    \STATE Update $z_k \gets \frac{1}{p_k} \nabla_{\theta} f(\theta_k; x_k, y_k)$
    \STATE Increment $m_k \gets m_k + \frac{1}{p_k}$
    \STATE Compute $\theta_k \gets \smoothop_{m_k}(z_k)$
    \STATE Average $\avtheta_k \gets \left( 1 - \frac{1}{m_k p_k} \right) \avtheta_k$
    \ENDFOR
    \ENDFOR
    \STATE \textbf{return} Each node $k$ has $\avtheta_k$
  \end{algorithmic}
\end{algorithm}

For any variant of gradient descent over a network with a decreasing step size, there is a need for a common time scale to perform the suitable decrease. In the synchronous setting, this time scale information can be shared easily among nodes by assuming the availability of a global clock. This is convenient for theoretical considerations, but is unrealistic in practical (asynchronous) scenarios.
In this section, we place ourselves in a fully asynchronous setting where each node has a local clock, ticking at a Poisson rate of $1$, independently from the others. This is equivalent to a global clock ticking at a rate $n$ Poisson process which wakes up an edge of the network uniformly at random \citep[see][for details on clock modeling]{Boyd2006a}.

With this in mind, Algorithm~\ref{alg:gossip_dual_averaging_sync} needs to be adapted to this setting. First, one cannot perform a full dual averaging update over the network since only two nodes wake up at each iteration. Also, as mentioned earlier, each node needs to maintain an estimate of the current iteration number in order for the scaling factor $\gamma$ to be consistent across the network. For $k \in [n]$, let $p_k$ denote the probability for the node $k$ to be picked at any iteration. If the edges are picked uniformly at random, then one has $p_k = 2 d_k / |E|$. For simplicity, we focus only on this case, although our analysis holds in a more general setting.

Let us define an activation variable $(\delta_k(t))_{t \geq 1}$ such that for any $t \geq 1$,
\[
  \delta_k(t) =
  \begin{cases}
    1 & \text{if node $k$ is picked at iteration $t$}, \\
    0 & \text{otherwise}.
  \end{cases}
\]
One can immediately see that $(\delta_k(t))_{t \geq 1}$ are i.i.d.\ random variables, Bernoulli distributed with parameter $p_k$. Let us define $(m_k(t)) \geq 0$ such that $m_k(0) = 0$ and for $t \geq 0$, $m_k(t + 1) = m_k(t) + \frac{\delta_k(t + 1)}{p_k}$. Since $(\delta_k(t))_{t \geq 1}$ are Bernoulli random variables, $m_k(t)$ is an unbiased estimate of the time $t$.

Using this estimator, we can now adapt Algorithm~\ref{alg:gossip_dual_averaging_sync} to the fully asynchronous case, as shown in Algorithm~\ref{alg:gossip_dual_averaging_async}.
The update step slightly differs from the synchronous case: the partial gradient has a weight $1/p_k$ instead of $1$ so that all partial functions asymptotically count in equal way in every gradient accumulator. In contrast, uniform weights would penalize partial gradients from low degree nodes since the probability of being drawn is proportional to the degree. This weighting scheme is essential to ensure the convergence to the global solution. The model averaging step also needs to be altered: in absence of any global clock, the weight $1/t$ cannot be used and is replaced by $1 / (m_k p_k)$, where $m_k p_k$ corresponds to the average number of times that node $k$ has been selected so far.

The following result is the analogous of Theorem~\ref{thm:gossip_dual_averaging_general_rate} for the asynchronous setting.

\begin{theorem}\label{th:asynch}
  Let $\mathcal{G}$ be a connected and non bipartite graph.
  Let $(\gamma(t))_{t \geq 1}$ be defined as $\gamma(t)= c/ t^{1/2+\alpha}$ for some constant $c>0$ and $\alpha \in(0,1/2)$. For $i \in [n]$, let $(d_i(t))_{t \geq 1}$, $(g_i(t))_{t \geq 1}$, $(\epsilon_i(t))_{t \geq 1} $, $(z_i(t))_{t \geq 1} $ and $(\theta_i(t))_{t \geq 1} $ be generated as described in Algorithm~\ref{alg:gossip_dual_averaging_async}. Then, there exists some constant $C<+\infty$ such that, for $\theta^* \in \argmin_{\theta' \in \bbR^d} R_n(\theta')$, $i \in [n]$ and $T > 0$, 
  \begin{align}
    R_n(\avtheta_i(T)) - R_n(\theta^*) \leq& C  \max(T^{-\alpha/2},T^{\alpha-1/2}) \nonumber\\
                                           &+ \frac{1}{T} \sum_{t = 2}^T \mathbb{E}_t[(\omega(t) - \theta^*)^{\top} \overline{\epsilon}^n(t)]. \nonumber
  \end{align}
\end{theorem}
The proof is given in the supplementary material.

\begin{remark}
  In the asynchronous setting, no convergence rate was known even for the distributed dual averaging algorithm of \citet{Duchi2012a}, which deals with the simpler problem of minimizing \emph{univariate} functions. The arguments used to derive Theorem~\ref{th:asynch} can be adapted to derive a convergence rate (without the bias term) for an asynchronous version of their algorithm.
\end{remark}

\begin{remark}
We have focused on the setting where all pairs of observations are involved in the objective. In practice, the objective may depend only on a subset of all pairs. To efficiently apply our algorithm to this case, one should take advantage of the potential structure of the subset of interest: for instance, one could attach some additional concise information to each observation so that a node can easily identify whether a pair contributes to the objective, and if not set the loss to be zero. This is essentially the case in the AUC optimization problem studied in Section~\ref{sec:experiments}, where pairs of similarly labeled observations do not contribute to the objective. If the subset of pairs cannot be expressed in such a compact form, then one would need to provide each node with an index list of active pairs, which could be memory-intensive when $n$ is large.
\end{remark}




\section{Numerical Simulations}
\label{sec:experiments}

\begin{table*}[t]
\centering
\small
\ra{1.3}
\begin{tabular}{@{}lcccc@{}}
\toprule
Dataset & Complete graph & Watts-Strogatz & Cycle graph \\
\midrule
Breast Cancer (AUC Maximization, $n = 699$) & $1.43 \cdot 10^{-3}$ & $8.71 \cdot 10^{-5}$ & $5.78 \cdot 10^{-8}$ \\
Synthetic (Metric Learning, $n = 1000$) & $1.00 \cdot 10^{-3}$ & $6.23 \cdot 10^{-5}$ & $1.97 \cdot 10^{-8}$ \\
\bottomrule
\end{tabular}
\caption{Spectral gap values $1 - \lambda_2^{\mathcal{G}}$ for each network.}
\label{tab:networks}
\end{table*}

In this section, we present numerical experiments on two popular machine learning problems involving pairwise functions: Area Under the ROC Curve (AUC) maximization and metric learning. Our results show that our algorithms converge and that the bias term vanishes very quickly with the number of iterations.

To study the influence of the network topology, we perform our simulations on three types of network (see Table~\ref{tab:networks} for the corresponding spectral gap values):
\begin{itemize}
\item \emph{Complete graph:} All nodes are connected to each other. It is the ideal situation in our framework, since any pair of nodes can communicate directly. In this setting, the bias of gradient estimates should be very small, as one has for any $k \in [n]$ and any $t \geq 1$, $\bbE_t[d_k(t)|\theta_k(t)] = 1/ (n - 1) \sum_{y' \neq y_k(t)} \nabla_{\theta} f(\theta_k(t); x_k, y')$. For a network size $n$, the complete graph achieves the highest spectral gap: $1 - \lambda_2^{\mathcal{G}}=1/n$, see \citet[][Ch.9]{bollobas1998modern} or \citet[Ch.1]{chung1997spectral} for details.
\item \emph{Cycle graph:} This is the worst case in terms of connectivity: each node only has two neighbors. This network has a spectral gap of order $1/n^{3}$, and gives a lower bound in terms of convergence rate.
\item \emph{Watts-Strogatz:} This random network generation technique \citep{watts1998collective} 
relies on two parameters: the average degree of the network $k$ and a rewiring probability $p$. In expectation, the higher the rewiring probability, the better the connectivity of the network. Here, we use $k = 5$ and $p = 0.3$ to achieve a compromise between the connectivities of the complete graph and the cycle graph.
\end{itemize}

\begin{figure*}[t]
  \centering
  \subfigure[Evolution of the objective function and its standard deviation (synchronous)]{\includegraphics[height=4.7cm]{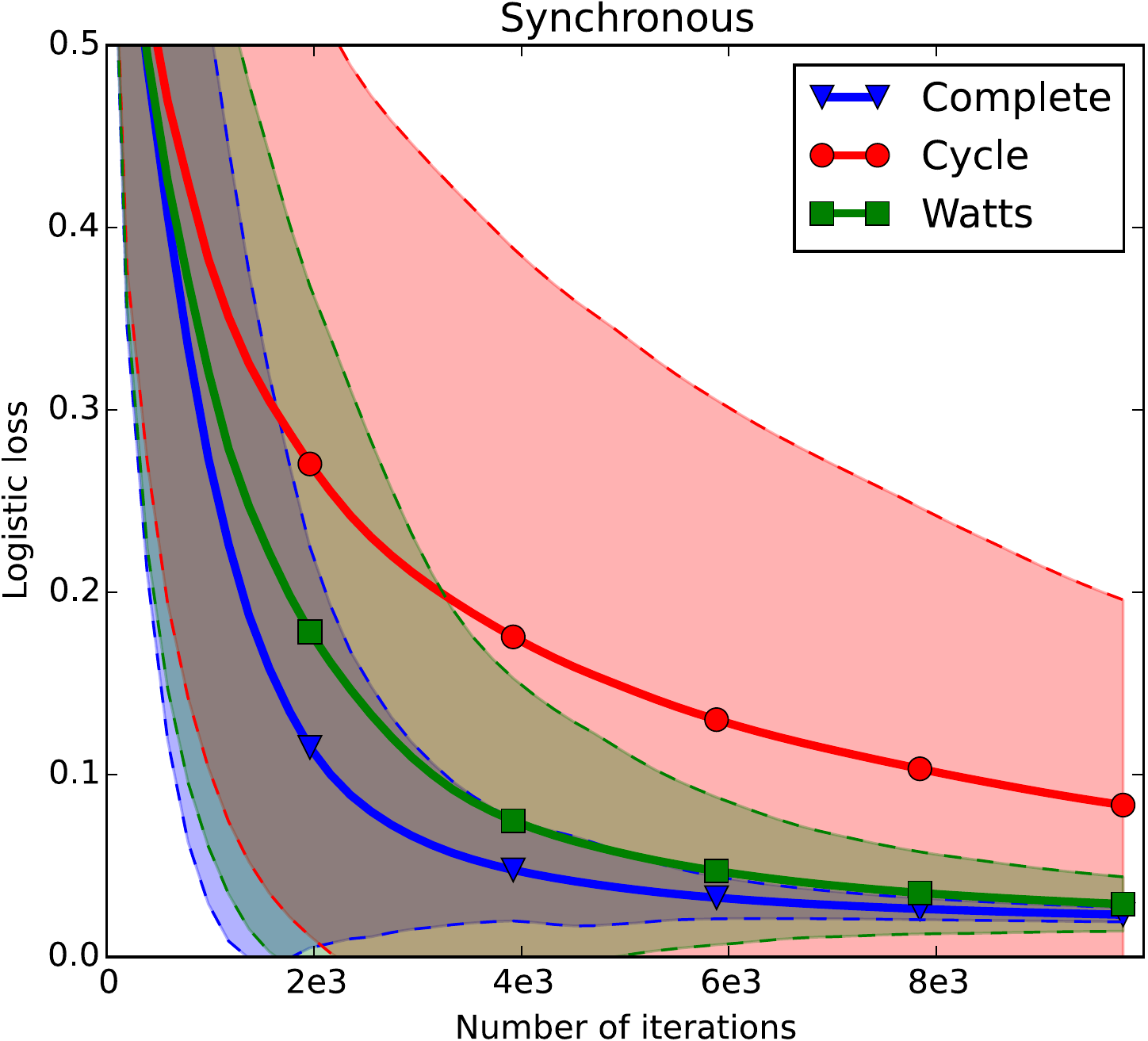}\label{fig:sync_graph-comp}}\hspace*{0.5cm}
  \subfigure[Evolution of the objective function and its standard deviation (asynchronous)]{\includegraphics[height=4.7cm]{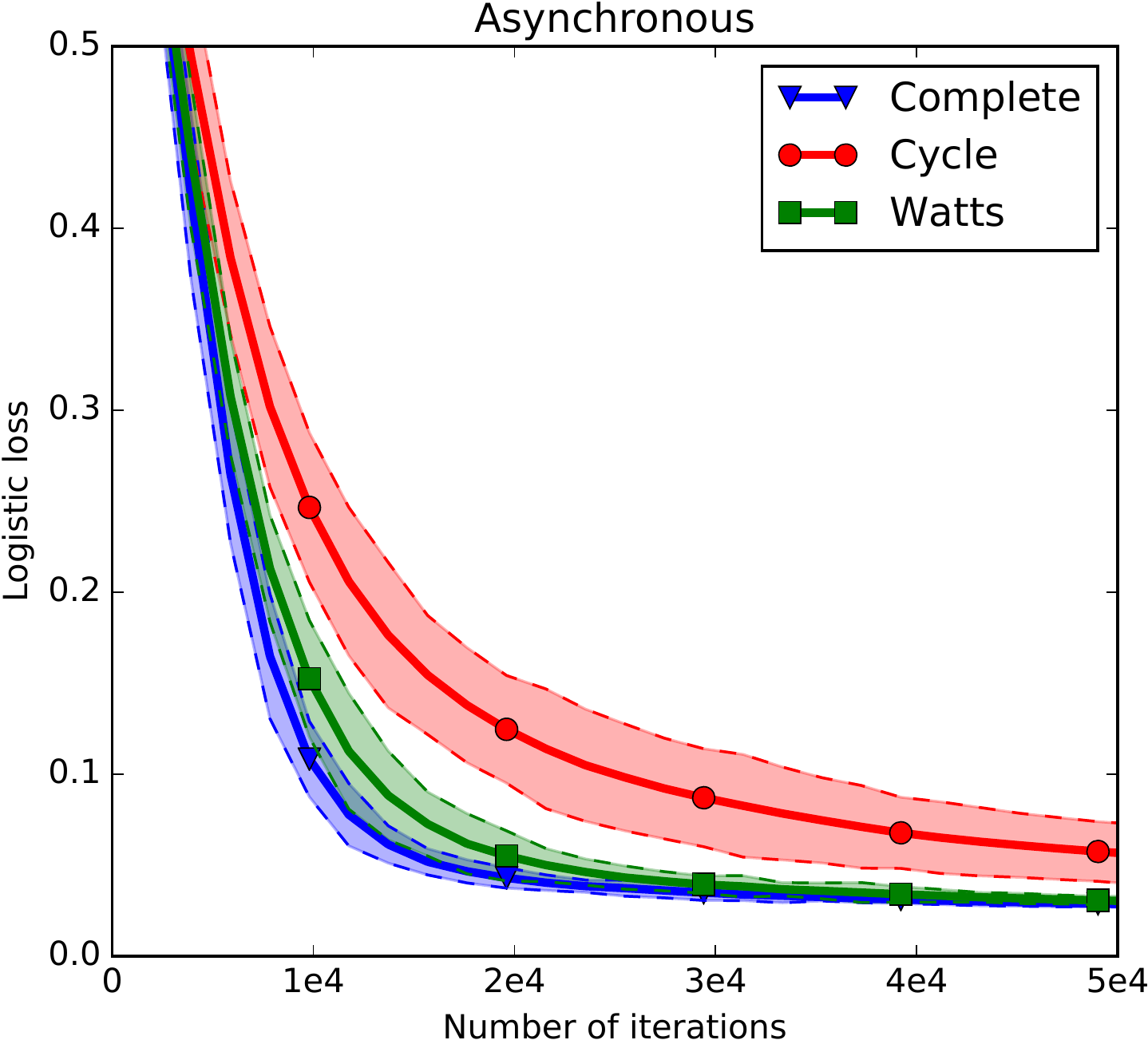}\label{fig:async_auc_standard}}\hspace*{0.5cm}
  \subfigure[Evolution of the bias term (asynchronous)]{\includegraphics[height=4.7cm]{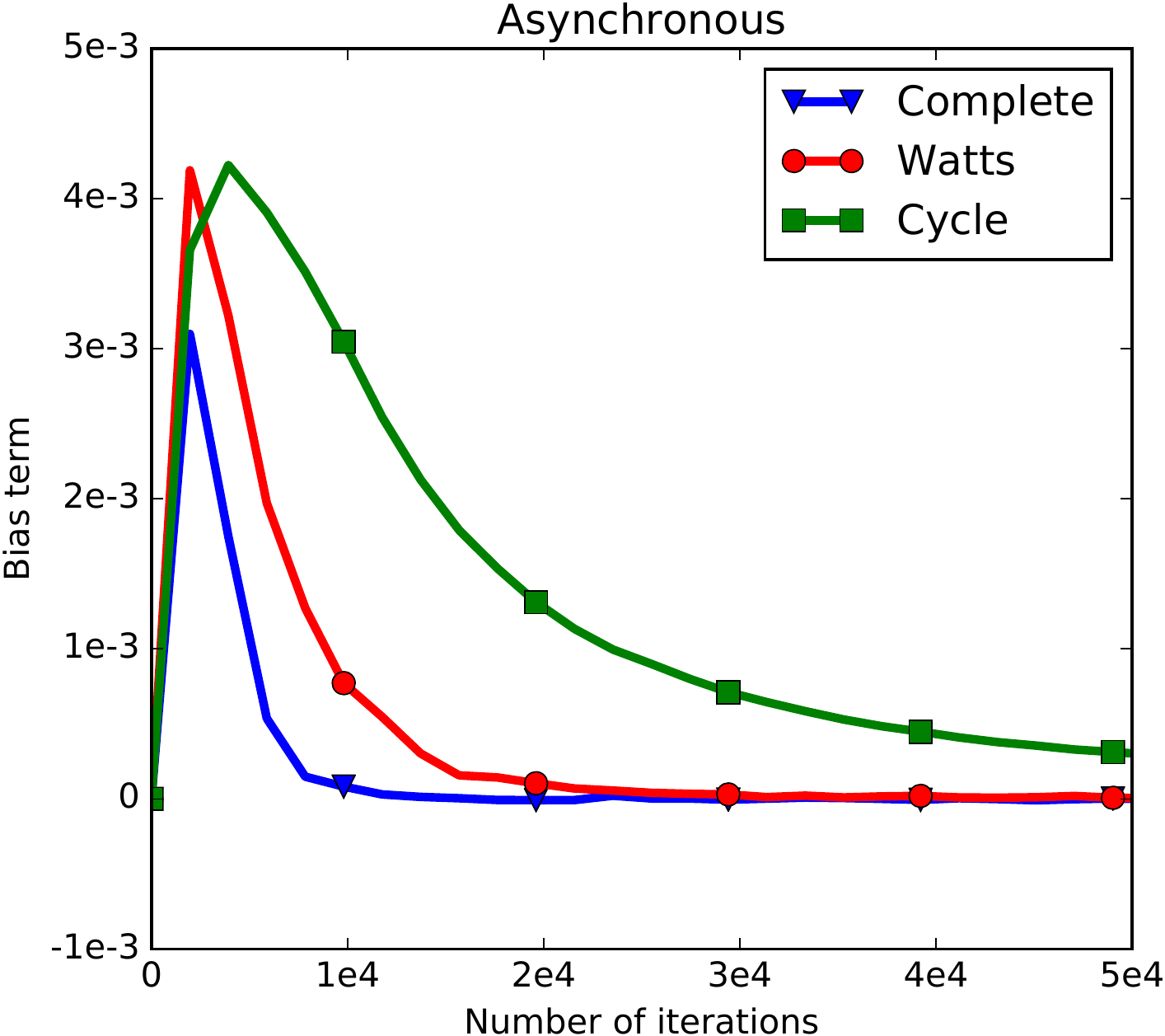}\label{fig:async_auc_bias}}
  \caption{AUC maximization in synchronous and asynchronous settings.}
\end{figure*}

\paragraph{AUC Maximization}
We first present an application of our algorithms to AUC maximization on a real dataset. Given a set of data points $x_1, \dots, x_n\in\bbR^d$ with associated binary labels $\ell_1,\dots,\ell_n\in\{-1,1\}$, the goal is to learn a linear scoring rule $x\mapsto x^\top\theta$ parameterized by $\theta\in\bbR^d$ which maximizes:
\begin{equation*}
  AUC(\theta) = \frac{\sum_{1 \leq i, j \leq n} \mathbb{I}_{\{\ell_i > \ell_j\}} \mathbb{I}_{\{x_i^{\top} \theta > x_j^{\top} \theta\}}}{\sum_{1 \leq i, j \leq n} \mathbb{I}_{\{\ell_i > \ell_j\}}}.
\end{equation*}
It corresponds to the probability that the scoring rule associated with $\theta$ outputs a higher score on a positively labeled sample than on a negatively labeled one. This formulation leads to a non-smooth optimization problem; therefore, one typically minimizes a convex surrogate such as the logistic loss:
\begin{equation*}
R_n(\theta) = \frac{1}{n^2} \sum_{1 \leq i, j \leq n} \mathbb{I}_{\{\ell_i > \ell_j\}} \log\left(1+\exp((x_j-x_i)^\top\theta)\right).
\end{equation*}
We do not apply any regularization (\textit{i.e.}, $\psi \equiv 0$), and use the Breast Cancer Wisconsin dataset,\footnote{\url{https://archive.ics.uci.edu/ml/datasets/Breast+Cancer+Wisconsin+(Original)}} which consists of $n = 699$ points in $d = 11$ dimensions.

We initialize each $\theta_i$ to $0$ and for each network, we run 50 times Algorithms~\ref{alg:gossip_dual_averaging_sync} and~\ref{alg:gossip_dual_averaging_async} with $\gamma(t) = 1 / \sqrt{t}$.\footnote{Even if this scaling sequence does not fulfill the hypothesis of Theorem~\ref{th:asynch} for the asynchronous setting, the convergence rate is acceptable in practice.} 
Figure~\ref{fig:sync_graph-comp} shows the evolution of the objective function and the associated standard deviation (across nodes) with the number of iterations in the synchronous setting. As expected, the average convergence rate on the complete and the Watts-Strogatz networks is much better than on the poorly connected cycle network. The standard deviation of the node estimates also decreases with the connectivity of the network.


The results for the asynchronous setting are shown in Figure~\ref{fig:async_auc_standard}. As expected, the convergence rate is slower in terms of number of iterations (roughly $5$ times) than in the synchronous setting. Note however that much fewer dual averaging steps are performed: for instance, on the Watts-Strogatz network, reaching a $0.1$ loss requires $210,000$ (partial) gradient computations in the synchronous setting and only $25,000$ in the asynchronous setting. Moreover, the standard deviation of the estimates is much lower than in the synchronous setting. This is because communication and local optimization are better balanced in the asynchronous setting (one optimization step for each gradient accumulator averaged) than in the synchronous setting ($n$ optimization steps for $2$ gradient accumulators averaged).

The good practical convergence of our algorithm comes from the fact that the bias term $\overline{\epsilon}^n(t)^{\top} \omega(t)$ vanishes quite fast. Figure~\ref{fig:async_auc_bias} shows that its average value quickly converges to $0$ on all networks. Moreover, its order of magnitude is negligible compared to the objective function.
In order to fully estimate the impact of this bias term on the performance, we also compare our algorithm to the ideal but unrealistic situation where each node is given an unbiased estimate of its partial gradient: instead of adding $\nabla f(\theta_i(t); x_i, y_i(t))$ to $z_i(t)$, a node $i$ will add $\nabla f(\theta_i(t); x_i, x_j)$ where $j\in[n]$ is picked uniformly at random. As shown in Figure~\ref{fig:async_auc_baseline-vs-gossip}, the performance of both methods are very similar on well-connected networks.

\begin{figure*}[t]
  \centering
  \includegraphics[height=5.2cm]{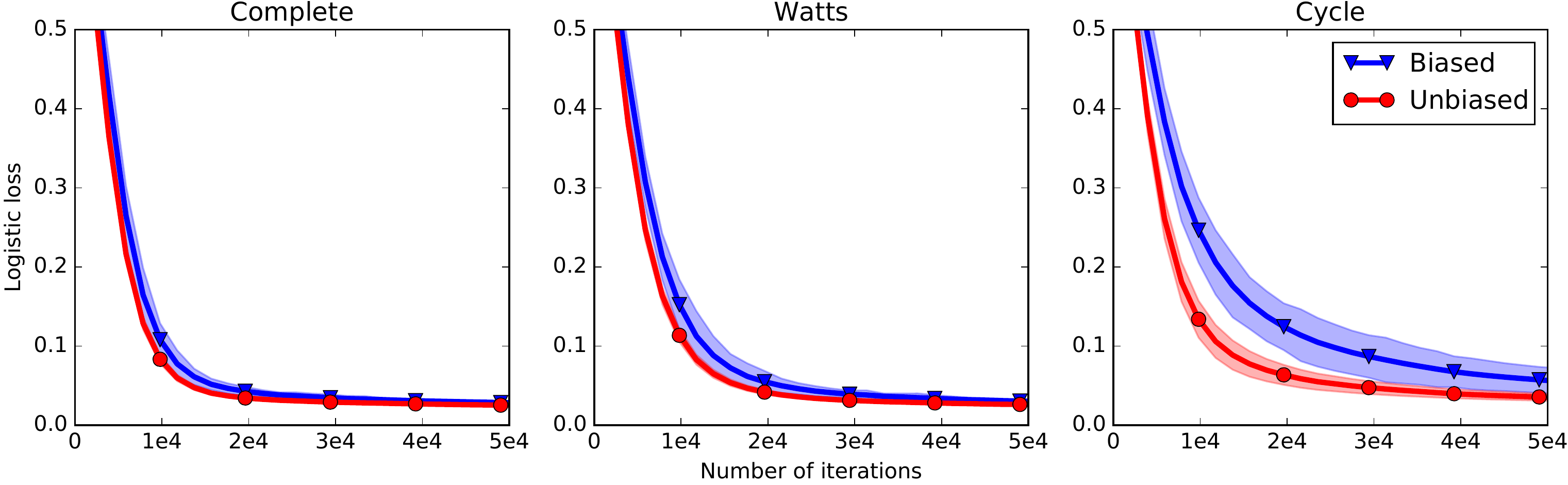}
  \caption{AUC maximization: comparison between our algorithm and an unbiased version.}
  \label{fig:async_auc_baseline-vs-gossip}
\end{figure*}

\begin{figure*}[t]
  \centering
\subfigure[Evolution of the objective function and its standard deviation (asynchronous setting)]{\hspace*{.5cm}\includegraphics[height=5.2cm]{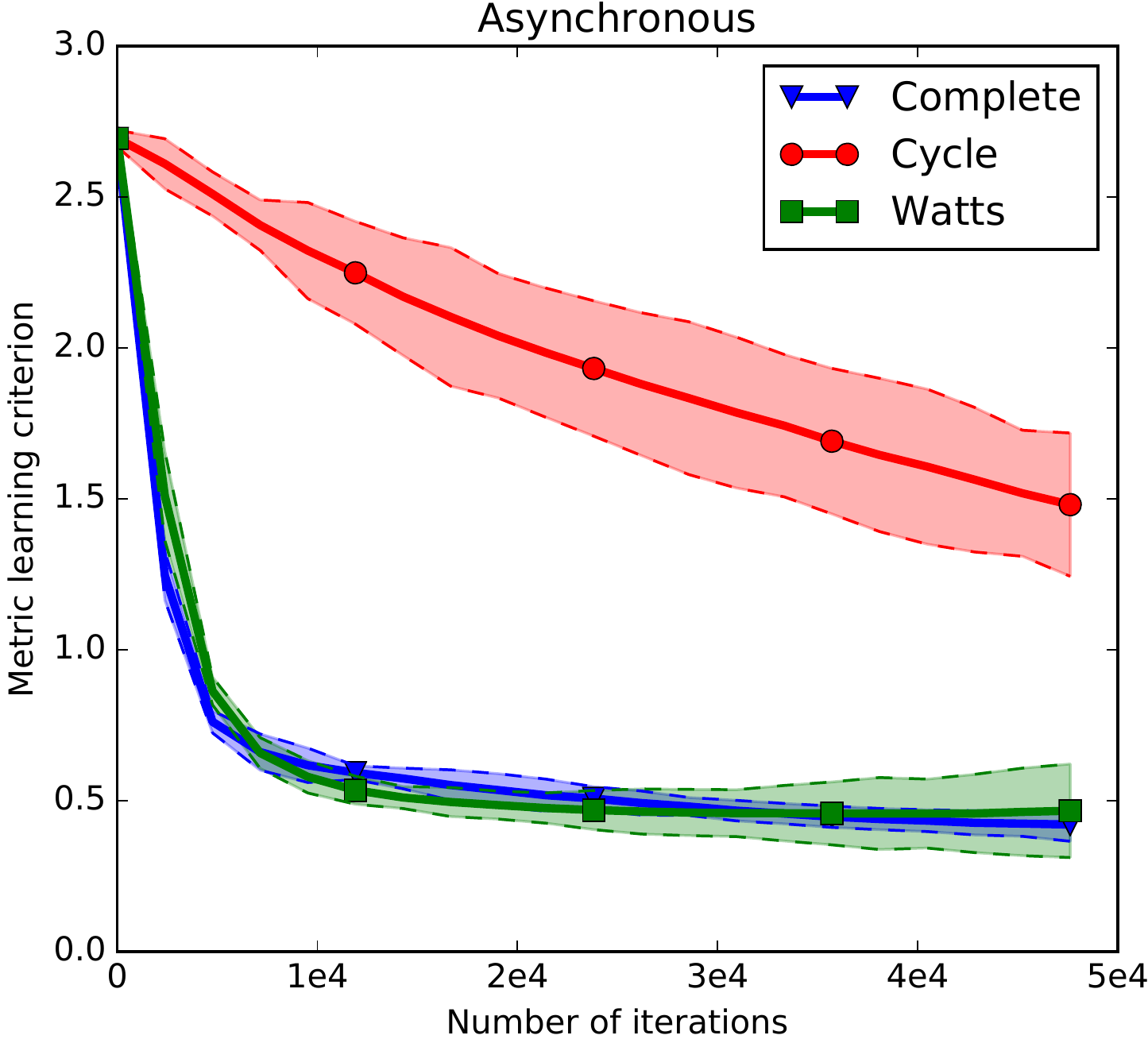}\hspace*{.5cm}\label{fig:async_ml_standard}}\hspace*{.5cm}
  \subfigure[Evolution of the bias term]{\hspace*{.5cm}\includegraphics[height=5.2cm]{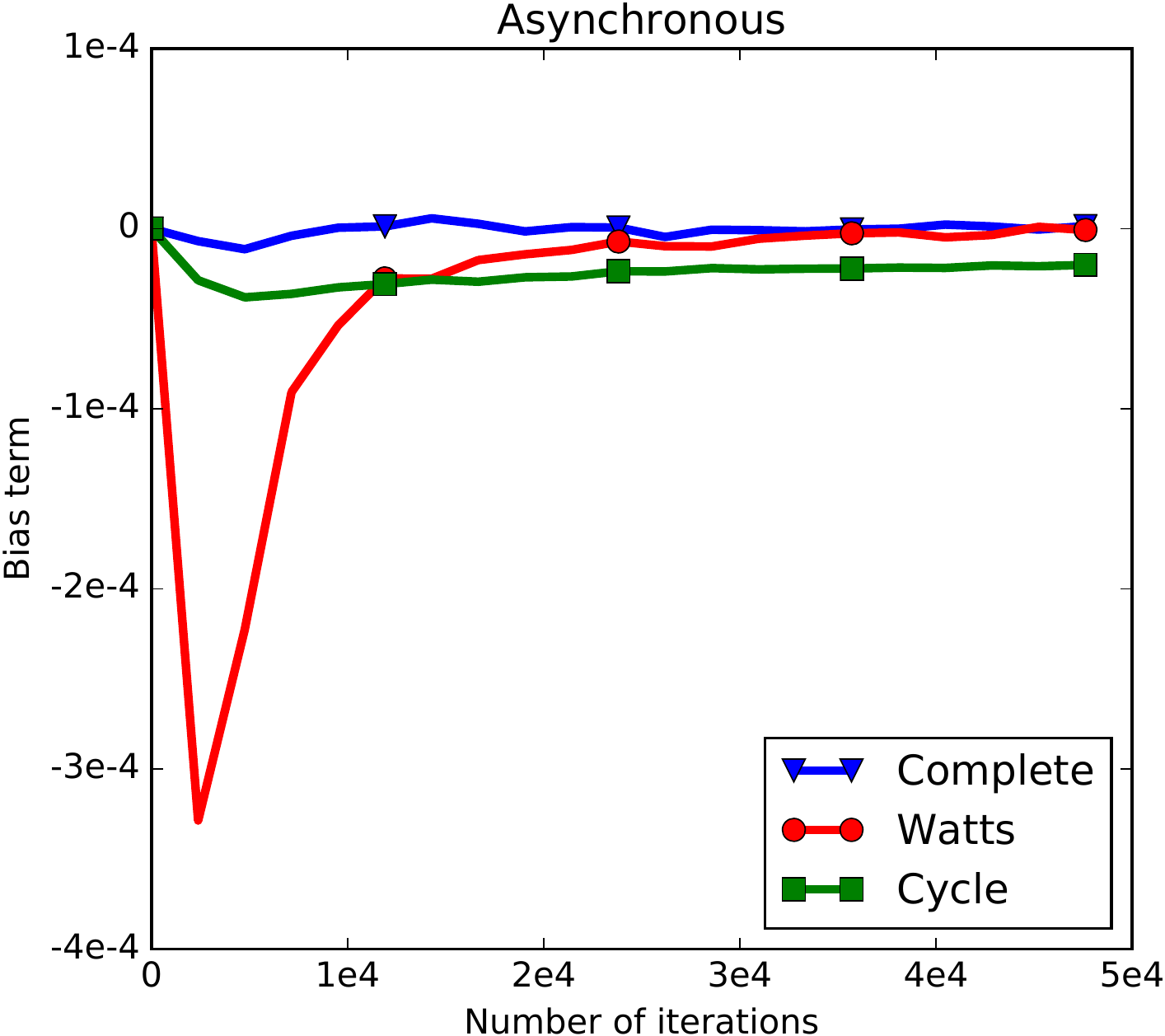}\hspace*{.5cm}\label{fig:async_ml_bias}}
  \caption{Metric learning experiments.}
\end{figure*}

\paragraph{Metric Learning}
We now turn to a metric learning application. We consider the family of Mahalanobis distances $D_\theta(x_i,x_j) = (x_i - x_j)^{\top} \theta (x_i - x_j)$ parameterized by $\theta\in\mathbb{S}_+^d$, where $\mathbb{S}_+^d$ is the cone of $d\times d$ positive semi-definite real-valued matrices. 
Given a set of data points $x_1, \dots, x_n\in\bbR^d$ with associated labels $\ell_1,\dots,\ell_n\in\{-1,1\}$, the goal is to find $\theta\in\mathbb{S}_+^d$ which minimizes the following criterion \citep{Jin2009a}:
$$R_n(\theta) = \frac{1}{n^2} \sum_{1 \leq i, j \leq n} \big[\ell_i \ell_j (b - D_\theta(x_i, x_j)) \big]_+ + \psi(\theta),$$
where $[u]_+=\max(0,1-u)$, $b > 0$, and $\psi(\theta)=\infty$ if $\theta\notin\mathbb{S}_+^d$ and $0$ otherwise.
We use a synthetic dataset of $n=1,000$ points generated as follows: each point is drawn from a mixture of $10$ Gaussians in $\mathbb{R}^{40}$ (each corresponding to a class) with all Gaussian means contained in a 5d subspace and their shared covariance matrix proportional to the identity with a variance factor such that some overlap is observed. 

Figure~\ref{fig:async_ml_standard} shows the evolution of the objective function and its standard deviation for the asynchronous setting. As in the case of AUC maximization, the algorithm converges much faster on the well-connected networks than on the cycle network. Again, we can see in Figure~\ref{fig:async_ml_bias} that the bias vanishes very quickly with the number of iterations.

\paragraph{Additional Experiment}
We refer to the supplementary material for a metric learning experiment on a real dataset. 


\vspace{-.15cm}
\section{Conclusion}
\label{sec:conclusion}

In this work, we have introduced new synchronous and asynchronous gossip algorithms to optimize functions depending on pairs of data points distributed over a network. The proposed methods are based on dual averaging and can readily accommodate various popular regularization terms. We provided an analysis showing that they behave similarly to the centralized dual averaging algorithm, with additional terms reflecting the network connectivity and the gradient bias. Finally, we proposed some numerical experiments on AUC maximization and metric learning which illustrate the performance of the proposed algorithms, as well as the influence of network topology.
A challenging line of future research consists in designing and analyzing novel adaptive gossip schemes, where the communication scheme is dynamic and depends on the network connectivity properties and on the local information carried by each node.


\onecolumn
\newpage
\appendix


\section{Outline of the Supplementary Material}
\label{sec:outline}

The supplementary material is organized as follows.
In Section~\ref{sec:centr-dual-aver}, we recall the standard proof of convergence rate for the (centralized) dual averaging.
Then, in Section~\ref{sec:decentr-with-bias}, we improve the analysis of the decentralized version of the dual averaging algorithm for simple sums of functions, and provide insights to analyze the case of sum of pairwise functions. Our asynchronous variant is investigated in Section~\ref{sec:asynchr-distr-sett}. Technical details on how to extend our framework to the case with multiple points per node are given in Section~\ref{app:multi}. Finally, additional numerical results are discussed in Section~\ref{sec:addit-exper}.

\section{Centralized Dual Averaging}
\label{sec:centr-dual-aver}

\subsection{Deterministic Setting}
\label{subsec:determ-sett}


We introduce the dual averaging algorithm for minimizing the sum $f+\psi$, in a context where $f$ is convex and smooth, $\psi(0)=0$, $\psi$ is convex, non-negative and possibly non-smooth, with a proximity operator simple to compute.  In the centralized framework, this algorithm reads as follows:
\begin{align}\label{eq:def_dual_av}
\theta(t + 1) = \argmin_{\theta' \in \bbR^d } \left\{ \theta'^{\top} \sum_{s = 1}^t g(s) + \frac{\| \theta' \|^2}{2 \gamma(t)} + t \psi(\theta') \right\},
\end{align}

for any $t \geq 1$, where $\gamma(t)$ represents a scale factor similar to a gradient step size use in standard gradient descent algorithms, and $g(t)$ is a sequence of gradient of $f$ taken at $\theta(t)$. Moreover we initialize $\theta(1)=0$. The function $f$ we consider is here of the form $\avf(\theta)=1/n \sum_{i=1}^n f_i(\theta)$, where each $f_i$ is assumed $L_f$-Lipschitz for simplicity (so is $ f$ then). We denote $R_n = \avf + \psi$. As a reminder, note that the Centralized dual averaging method is explicitly stated in Algorithm~\ref{alg:dual_averaging_standard}.

This particular formulation was introduced in \citep{xiao2009dual, Xiao10}, extending the method introduced by \cite{Nesterov2009a} in the specific case of indicator functions. In this work, we borrow the notation from \cite{Xiao10}.

In order to perform a theoretical analysis of this algorithm, we introduce the following functions. Let us define, for $t \geq 0$
\[
V_t(z) := \displaystyle \max_{\theta \in \bbR^d} \left\{ z^{\top} \theta - \frac{\| \theta \|^2}{2 \gamma(t)}  - t \psi(\theta) \right\}
.
\]
Remark that with the assumption that $\psi(0)=0$, then $V_t(0) = 0$.
We also define the smoothing function $\smoothop_t$ that plays a crucial role in the dual algorithm formulation:
\[
    \smoothop_t(z) :=  \argmax_{\theta \in \bbR^d}
    \left\{
    z^{\top} \theta - \frac{\| \theta \|^2}{2 \gamma(t)}  - t \psi(\theta) \right\}=
    \argmin_{\theta \in \bbR^d}
    \left\{
    -z^{\top}\theta + \frac{\| \theta \|^2}{2 \gamma(t)}  + t \psi(\theta) \right\}
  \]

 Strong convexity in $\theta$ of the objective function, ensures that the solution of the optimization problem is unique. The following lemma links the function $V_t$ and the algorithm update and is a simple application of the results from \citep[Lemma 10]{xiao2009dual}:

\begin{lemma}
  \label{lma:grad_argmin}
  For any $z \in \bbR^d $, one has:
 \begin{align}
    \smoothop_t(z) = \nabla V_t (z) \,,
 \end{align}
   and the following statements hold true: for any $z_1,z_2 \in \bbR^d$
  \begin{align}
      \|    \smoothop_t(z_1)-\smoothop_t(z_2)\| \leq \gamma(t)\|z_1-z_2\| \,,
  \end{align}
and for any $g,z \in \bbR^d$,
\begin{align}\label{ineq:stongly_cvx}
V_t(z+g) \leq V_t(z) + g^\top \nabla V_t(z) + \frac{\gamma(t)}{2} \|g\|^2.
\end{align}
\end{lemma}
With this notation one can write the dual averaging rule as $\theta(t + 1)=\smoothop_t \left(-z(t + 1) \right)$, where $z(t) := \sum_{s=1}^{t - 1} g(s)$, with the convention $z(1) = 0$. Moreover, adapting \citep[Lemma 11]{xiao2009dual} we can state:

\begin{lemma}
\label{lma:bound_vprox}
For any $t \geq 1$ and any non-increasing sequence $(\gamma(t))_{t \geq 1}$, we have
\begin{align}
V_t\left(-z(t + 1) \right) + \psi(\theta(t + 1)) \leq V_{t - 1}\left(-z(t + 1) \right).
\end{align}
\end{lemma}

\begin{algorithm}[t]
  \small
  \caption{Centralized dual averaging}
  \label{alg:dual_averaging_standard}
  \begin{algorithmic}[1]
    \REQUIRE Step size $(\gamma(t))_{t \geq 1} > 0$.
    \STATE Initialization $\theta = 0$, $\avtheta = 0$, $z = 0$.
    \FOR{$t = 1, \ldots, T$}
    \STATE Update $z \gets z + g(t)$, where $g(t) = \nabla \avf(\theta)$
    \STATE Update $\theta \gets \smoothop_t(z)$
    \STATE Update $\avtheta \gets \left( 1 - \frac{1}{t} \right) \avtheta + \frac{1}{t} \theta$
    \ENDFOR
    \RETURN $\avtheta$
  \end{algorithmic}
\end{algorithm}

We also need a last technical result that we will use several times in the following:

\begin{lemma}
\label{lma:bound_inner_pdct}

Let $\theta(t)=\smoothop_t (\sum_{s=1}^{t - 1} g(s))$, and let $(\gamma(t))_{t\geq1}$ be a non-increasing and non-negative sequence sequence (with the convention $\gamma(0)=0$), then for any $\theta \in \bbR^d$:

\begin{align}
    \frac{1}{T} \sum_{t = 1}^T g(t)^{\top} (\theta(t) - \theta) +
    \frac{1}{T}\sum_{t = 1}^T (\psi( \theta(t)) - \psi(\theta)) \label{ineq:cvx_dual_proof}
    \leq & \frac{1}{T} \sum_{t = 1}^T \frac{\gamma(t-1)}{2} \|g(t)\|^2+ \frac{\| \theta \|^2}{2T \gamma(T)} \,.
\end{align}

\begin{proof}
Use the definition of $V_T$ to get the following upper bound
\begin{align}
    \frac{1}{T} \sum_{t = 1}^T g(t)^{\top} (\theta(t) - \theta) +
    \frac{1}{T}\sum_{t = 1}^T (\psi( \theta(t)) - \psi(\theta))
    = &  \frac{1}{T} \sum_{t = 1}^T g(t)^{\top} \theta(t) + \psi( \theta(t))+ \frac{\| \theta \|^2}{2T \gamma(T)} - \psi(\theta) \nonumber\\
    & - \left(\frac{z(T + 1)}{T}\right)^{\top} \theta -\frac{\| \theta \|^2}{2T \gamma(T)}\nonumber\\
    \leq &  \frac{1}{T} \sum_{t = 1}^T  \left(  g(t)^{\top} \theta(t) + \psi( \theta(t)) \right) \nonumber\\
    &+ \frac{\| \theta^* \|^2}{2T \gamma(T)} + V_T(-z(T + 1)) \label{ineq:intermed_lemma} \,.
\end{align}
Then one can check that with \eqref{ineq:stongly_cvx} and Lemma~\ref{lma:bound_vprox} that:
\begin{align*}
V_t(-z(t + 1)) + \psi (\theta(t+1))
  \leq & V_{t-1}(-z(t + 1))\\
  = & V_{t-1}(-z(t)-g(t))\\
  \leq & V_{t-1}(-z(t)) - g(t)^\top \nabla V_{t-1}(-z(t)) + \frac{\gamma(t-1)}{2} \|g(t)\|^2\\
  = & V_{t-1}(-z(t)) - g(t)^\top \theta(t) + \frac{\gamma(t-1)}{2} \|g(t)\|^2.
\end{align*}
From the last display, the following holds:
\begin{align*}
g(t)^\top \theta(t) + \psi (\theta(t+1)) \leq
V_{t-1}(-z(t)) -V_t(-z(t + 1))   + \frac{\gamma(t-1)}{2} \|g(t)\|^2.
\end{align*}
Summing the former for $t =1,\ldots,T$ yields
\begin{align*}
\sum_{t=1}^T g(t)^\top \theta(t) + \psi (\theta(t+1)) \leq
V_{0}(-s_{0}) -V_T(-s_T) + \sum_{t=1}^T \frac{\gamma(t-1)}{2} \|g_{t}\|^2.
\end{align*}
Remark that $V_0(0) = 0$ and $\psi(\theta(1)) - \psi(\theta(T + 1)) = -\psi(\theta(T + 1))\leq 0$, so the previous display can be reduced to:

\begin{equation}
\label{ineq:bound_inner_product}
\sum_{t=1}^T g(t)^\top \theta(t) + \psi (\theta(t)) + V_T(-z(T + 1)) \leq \sum_{t=1}^T \frac{\gamma(t-1)}{2} \|g(t)\|^2.
\end{equation}
Combining with \eqref{ineq:intermed_lemma}, the lemma holds true.
\end{proof}

\end{lemma}
Bounding the error of the dual averaging is provided in the next theorem, where we remind that $R_n = \avf + \psi$:
\begin{theorem}\label{thm:dual_averaging_standard}
  Let $(\gamma(t))_{t \geq 1}$ be a non increasing sequence. Let $(z(t))_{t \geq 1}$, $(\theta(t))_{t \geq 1}$, $(\avtheta(t))_{t \geq 1}$ and $(g(t))_{t \geq 1}$ be generated according to Algorithm~\ref{alg:dual_averaging_standard}. Assume that the function $\avf$ is $L_f$-Lipschitz and that $\theta^* \in \argmin_{\theta' \in \bbR^d} R_n(\theta')$, then for any $T \geq 2 $, one has:
  \begin{equation}
    \label{eq:dual_standard_bound}
    R_n(\avtheta(T)) - R_n(\theta^*) \leq \frac{\| \theta^* \|^2}{2T \gamma(T)} + \frac{L_f^2}{2T} \sum_{t = 1}^{T - 1} \gamma(t).
  \end{equation}

  Moreover, if one knows $D > 0$ such that $\|\theta^*\| \leq D$, then for the choice $\gamma(t) = \frac{D}{L_f \sqrt{2t}}$, one has:
  \[
    R_n(\avtheta(T)) - R_n(\theta^*) \leq \frac{\sqrt{2} D L_f}{\sqrt{T}}.
  \]
\end{theorem}

\begin{proof}
  Let $T \geq 2 $. Using the convexity of $\avf$ and $\psi$, we can get:
\begin{align*}
    R_n(\avtheta(T)) - R_n(\theta^*)
    &\leq \frac{1}{T} \sum_{t = 1}^T \avf(\theta(t)) - \avf(\theta^*) +\psi( \avtheta) - \psi(\theta^*) \nonumber\\
    &\leq \frac{1}{T} \sum_{t = 1}^T g(t)^{\top} (\theta(t) - \theta^*) +
    \frac{1}{T}\sum_{t = 1}^T (\psi( \theta(t)) - \psi(\theta^*)) \\
    &\leq \frac{1}{T} \sum_{t = 1}^T \frac{\gamma(t-1)}{2} \|g(t)\|^2+ \frac{\| \theta \|^2}{2T \gamma(T)} \,.
\end{align*}
where the second inequality holds since $g(t) = \nabla \avf(\theta(t))$, and the third one is from an application of Lemma~\ref{lma:bound_inner_pdct} with the choice $\theta=\theta^*$.
Provided that $\|g(t)\|\leq L_f$, which is true whenever $\avf$ is $L_f$-Lipschitz.
\end{proof}




\subsection{Stochastic Dual Averaging}
\label{sec:stoch-subgr-dual}

Similarly to sub-gradient descent algorithms, one can adapt dual averaging algorithm to a stochastic setting; this was studied extensively by \citet{xiao2009dual}. Instead of updating the dual variable $z(t)$ with the (full) gradient of $\avf$ at $\theta(t)$, one now only requires the \emph{expected} value of the update to be the gradient, as detailed in Algorithm~\ref{alg:dual_averaging_sto}.

As in the gradient descent case, convergence results still hold in expectation, as stated in Theorem~\ref{thm:dual_av_sto}.
\begin{theorem}
  \label{thm:dual_av_sto}
  Let $(\gamma(t))_{t \geq 1} $ be a non increasing sequence. Let $(z(t))_{t \geq 1} $, $(\theta(t))_{t \geq 1} $ and $(g(t))_{t \geq 1} $ be generated according to Algorithm~\ref{alg:dual_averaging_sto}. Assume that the function $\avf$ is $L_f$-Lipschitz and that $\theta^* \in \argmin_{\theta' \in \bbR^d} R_n(\theta')$, then for any $T \geq 2 $, one has:
  \begin{equation}
    \label{eq:dual_sto_bound}
    \bbE_T\Big[ R_n(\avtheta(T)) - R_n(\theta^*) \Big] \leq \frac{\| \theta^* \|^2}{2T \gamma(T)} + \frac{L_f^2}{2T} \sum_{t = 1}^{T - 1} \gamma(t),
  \end{equation}
where $\bbE_T$ is the expectation over all possible sequence $(g(t))_{1 \leq t \leq T}$.

Moreover, if one knows that $D > 0$ such that $\|\theta^*\| \leq D$, then for $\gamma(t) = \frac{D}{L_f \sqrt{2t}}$, one has:
  \[
       \bbE_T  \big[R_n(\avtheta(T)) - R_n(\theta^*) \big] \leq \frac{\sqrt{2} D L_f}{\sqrt{T}}.
  \]
\end{theorem}
\begin{proof}
One only has to prove that the convexity inequality in Lemma~\ref{lma:bound_inner_pdct} holds in expectation. The rest of the proof can be directly adapted from Theorem~\ref{thm:dual_averaging_standard}.

Let $T \geq 2 $; using the convexity of $\avf$, one obtains:
\[
  \bbE_T[\avf(\avtheta(T)) - \avf(\theta^*)] \leq \frac{1}{T} \sum_{t = 1}^T \bbE_T[\avf(\theta(t)) - \avf(\theta^*)].
\]
For any $0 < t \leq T$, $\bbE[\theta(t) | g(0), \ldots, g(t - 1)] = \theta(t)$. Therefore, we have:
\[
  \bbE_T[\avf(\theta(t)) - \avf(\theta^*)] = \bbE_{t - 1}[\avf(\theta(t)) - \avf(\theta^*)].
\]
The vector $\bbE_t[g(t) | \theta(t)]$ is the gradient of $\avf$ at $\theta(t)$, we can then use $\avf$ convexity to write:
\[
  \bbE_{t - 1}[\avf(\theta(t)) - \avf(\theta^*)] \leq \bbE_{t - 1}\Big[(\theta(t) - \theta^*)^{\top} \bbE_t[g(t) | \theta(t)] \Big].
\]
Using properties of conditional expectation, we obtain:
\[
  \bbE_{t - 1}\Big[(\theta(t) - \theta^*)^{\top} \bbE_t[g(t) | \theta(t)] \Big] = \bbE_{t - 1}\Big[\bbE_t[(\theta(t) - \theta^*)^{\top} g(t) | \theta(t)] \Big] = \bbE_t[(\theta(t) - \theta^*)^{\top} g(t)].
\]
Finally, we can write:
\begin{equation}
  \label{ineq:sto_cvx}
  \bbE_T[\avf( \avtheta(T) - f(\theta^*)] \leq \frac{1}{T} \sum_{t = 1}^T \bbE_t[(\theta(t) - \theta^*)^{\top} g(t)] = \bbE_T \left[ \frac{1}{T} \sum_{t = 1}^T (\theta(t) - \theta^*)^{\top} g(t) \right].
\end{equation}
\end{proof}

\section{Convergence Proof for Synchronous Pairwise Gossip Dual Averaging}
\label{sec:decentr-with-bias}

In \cite{Duchi2012a}, the following convergence rate for distributed dual averaging is established:
\begin{align*}
  R_n(\avtheta_i(T)) - R_n(\theta^*)
  &\leq \frac{1}{2 T \gamma(T)} \| \theta^* \|^2 + \frac{L_f^2}{2T} \sum_{t = 2}^T \gamma(t - 1) \\
  &+ \frac{L_f}{nT} \sum_{t = 2}^T \gamma(t - 1) \sum_{j = 1}^n \Big( \|z_i(t) - z_j(t) \| + \| \avz(t) - z_j(t) \| \Big).
\end{align*}

The first part is an optimization term, which is exactly the same as in the centralized setting. Then, the second part is a network-dependent term which depends on the global variation of the dual variables; the following lemma provides an explicit dependence between this term and the topology of the network.
\begin{lemma}
  \label{lma:control_on_z}
  Let $W(\mathcal{G}) = I_n - \frac{L(\mathcal{G})}{|E|}$ and let $(G(t))_{t \geq 1}$ and $(Z(t))_{t \geq 1}$ respectively be the gradients and the gradients cummulative sum of the distributed dual averaging algorithm. If $\mathcal{G}$ is connected and non bipartite, then one has for $t \geq 1$:
  \[
    \frac{1}{n} \sum_{i = 1}^n \bbE \|z_i(t) - \overline{z}^n(t) \| \leq \frac{L_f}{1 - \sqrt{\lambda_2^{\mathcal{G}}}},
  \]
  where $\lambda_2^{\mathcal{G}}$ is the second largest eigenvalue of $W(\mathcal{G})$.
\end{lemma}
\begin{proof}
  For $t \geq 1$, let $W(t)$ be the random matrix such that if $(i, j) \in E$ is picked at $t$, then
  \[
W(t) = I_n - \frac{1}{2}(e_i - e_j)(e_i - e_j)^{\top}.
\]
  As denoted in \cite{Duchi2012a}, the update rule for $Z$ can be expressed as follows:
  \[
    Z(t + 1) = G(t) + W(t)Z(t),
  \]
  for any $t \geq 1$, reminding that $G(0)=0, Z(1)=0$. Therefore, one can obtain recursively
  \[
    Z(t) = \sum_{s = 0}^t W(t:s) G(s),
  \]
  where $W(t:s) = W(t) \ldots W(s + 1)$, with the convention $W(t:t) = I_n$. For any $t \geq 1$, let $W'(t) := W(t) - \frac{\1_n \1_n^{\top}}{n}$. One can notice that for any $0 \leq s \leq t$, $W'(t:s) = W(t:s) - \frac{\1_n \1_n^{\top}}{n}$ and write:
  \[
    Z(t) - \1_n \overline{z}^n(t)^{\top} = \sum_{s = 0}^t W'(t:s)G(s).
  \]
  We now take the expected value of the Frobenius norm:
  \begin{align}
    \bbE \left[ \left\| Z(t) - \1_n \overline{z}^n(t)^{\top} \right\|_F \right] 
    &\leq \sum_{s = 0}^t \bbE \left[ \left\| W(t:s) G(s) \right\|_F \right] \nonumber\\
    &\leq \sum_{s = 0}^t \sqrt{\bbE \left[ \left\| W(t:s) G(s) \right\|^2_F \right]} \nonumber\\
    &= \sum_{i = 1}^n \sum_{s = 0}^t \sqrt{\bbE \left[ g^{(i)}(s)^{\top} W'(t:s)^{\top} W'(t:s) g^{(i)}(s) \right]},\nonumber
  \end{align}
  where $g^{(i)}(s)$ is the column~$i$ of matrix $G(s)$. Since for any $s \geq 0$, $W(s)$ is a symmetric projection matrix, $W'(s)^{\top} W'(s) = W'(s)$; moreover, conditioning over $\mathcal{F}_s$ leads to:
  \begin{align}
    \bbE \left[ g^{(i)}(s)^{\top} W'(t:s)^{\top} W'(t:s) g^{(i)}(s) \right] = \bbE \left[ g^{(i)}(s)^{\top} \bbE[W'(t:s) | \mathcal{F}_s] g^{(i)}(s) \right] \leq \lambda_2^{\mathcal{G}} \|g^{(i)}(s)\|^2.
  \end{align}
  Using the fact that for any $s \geq 0$, $\| G(s) \|_F^2 \leq n L^2_f$, one has:
  \[
    \bbE \left[ \left\| Z(t) - \1_n \overline{z}^n(t)^{\top} \right\|_F \right] \leq \sqrt{n} L_f \sum_{s = 0}^t \left(\lambda_2^{\mathcal{G}}\right)^{\frac{t-s}{2}} \leq \frac{\sqrt{n}L_f}{1 - \sqrt{\lambda_2^{\mathcal{G}}}}.
  \]
  Finally, using the bounds between $\ell_1$ and $\ell_2$-norms yields:
  \[
    \frac{1}{n} \sum_{i = 1}^n \bbE \| z_i(t) - \overline{z}^n(t) \| \leq \frac{1}{\sqrt{n}} \bbE \left\| Z(t) - \1_n \overline{z}^n(t)^{\top} \right\|_F \leq \frac{L_f}{1 - \sqrt{\lambda_2^{\mathcal{G}}}}.
  \]

\end{proof}
With this bound on the dual variables, one can reformulate the convergence rate as stated below.
\begin{cor}
  Let $\mathcal{G}$ be a connected and non bipartite graph. Let $(\gamma(t))_{t \geq 1}$ be a non-increasing and non-negative sequence. For $i \in [n]$, let $(g_i(t))_{t \geq 1} $, $(z_i(t))_{t \geq 1} $ and $(\theta_i(t))_{t \geq 1} $ be generated according to the distributed dual averaging algorithm. For $\theta^* \in \argmin_{\theta' \in \bbR^d} R_n(\theta')$, $i \in [n]$ and $T \geq 2$, one has:
  \begin{align*}
    R_n(\avtheta_i(T)) - R_n(\theta^*)
    \leq& \frac{1}{2 T \gamma(T)} \| \theta^* \|^2 + \frac{L_f^2}{2T} \sum_{t = 1}^{T - 1} \gamma(t) \\
    &+ \frac{3L_f^2}{T\left(1 - \sqrt{\lambda_2^{\mathcal{G}}}\right)} \sum_{t = 1}^{T - 1} \gamma(t),
  \end{align*}
  where $\lambda_2^{\mathcal{G}} < 1$ is the second largest eigenvalue of $W(\mathcal{G})$.
\end{cor}

We now focus on gossip dual averaging for pairwise functions, as shown in Algorithm~\ref{alg:gossip_dual_averaging_sync}. The key observation is that, at each iteration, the descent direction is stochastic but also a \emph{biased} estimate of the gradient. That is, instead of updating a dual variable $z_i(t)$ with $g_i(t)$ such that $\bbE[g_i(t) | \theta_i(t)] = \nabla f_i(\theta_i(t))$, we perform some update $d_i(t)$, and we denote by $\epsilon_i(t)$ the quantity such that $\bbE[d_i(t) - \epsilon_i(t) | \theta_i(t)] = \bbE[g_i(t) | \theta_i(t)] = \nabla f_i(\theta_i(t))$. The following theorem allows to upper-bound the error induced by the bias.

\begin{theorem}
  \label{thm:dist_dual_averaging_stoch_biased_rate}
  Let $\mathcal{G}$ be a connected and non bipartite graph.
  Let $(\gamma(t))_{t \geq 1}$ be a non increasing and non-negative sequence. For $i \in [n]$, let $(d_i(t))_{t \geq 1}$, $(g_i(t))_{t \geq 1} $, $(\epsilon_i(t))_{t \geq 1} $, $(z_i(t))_{t \geq 1} $ and $(\theta_i(t))_{t \geq 1}$ be generated by Algorithm~\ref{alg:gossip_dual_averaging_sync}. Assume that the function $\avf$ is $L$-Lipschitz and that $\theta^* \in \argmin_{\theta' \in \bbR^d} R_n(\theta')$, then for any $i \in [n]$ and $T \geq 2$, one has:
  \begin{align}
    \bbE_T [R_n(\avtheta_i(T))] - R_n(\theta^*)
    &\leq \frac{1}{2 T \gamma(T)} \| \theta^* \|^2 + \frac{L_f^2}{2T} \sum_{t = 1}^{T - 1} \gamma(t) \nonumber\\
    &+ \frac{3L_f^2}{T\left(1 - \sqrt{\lambda_2^{\mathcal{G}}}\right)} \sum_{t = 1}^{T - 1} \gamma(t) \nonumber\\
    &+ \frac{1}{T} \sum_{t = 1}^{T - 1} \bbE_t[(\omega(t) - \theta^*)^{\top} \bar{\epsilon}^n(t)]. \nonumber
  \end{align}
\end{theorem}
\begin{proof}
 We can apply the same arguments as in the proofs of centralized and distributed dual averaging, so for $T > 0$ and $i \in [n]$:
  \begin{align*}
      \bbE_T[ R_n(\avtheta_i(T))] - R_n(\theta^*)
      &\leq \frac{L}{nT} \sum_{t = 2}^T \gamma(t - 1) \sum_{j = 1}^n \bbE\Big[ \|z_i(t) - z_j(t) \| + \| \avz(t) - z_j(t) \| \Big] \\
      &+ \frac{1}{T} \sum_{t = 2}^T \bbE_t[(\omega(t) - \theta^*)^{\top} \avg(t) ].
  \end{align*}
   However, Lemma~\ref{lma:bound_inner_pdct} can no longer be applied here since the updates are performed with $d_j(t)$ and not $g_j(t) = d_j(t) - \epsilon_j(t)$. With the definition of $d_j(t)$, the former yields:
\begin{align*}
      \frac{1}{T} \sum_{t = 2}^T \bbE_t[\omega(t) - \theta^*)^{\top} \bar{g}^n(t) ] &= \frac{1}{T} \sum_{t = 2}^T \bbE_t[(\omega(t) - \theta^*)^{\top} (\bar{d}^n(t) - \bar{\epsilon}^n(t))].
\end{align*}
Now Lemma~\ref{lma:bound_inner_pdct} can be applied to the first term in the right hand side and the result holds.
\end{proof}

\section{Asynchronous Distributed Setting}
\label{sec:asynchr-distr-sett}

In this section, we focus on a fully asynchronous setting where each node has a local clock. We assume for simplicity that each node has a clock ticking at a Poisson rate equals to $1$, so it is equivalent to a global clock ticking at a Poisson rate of $n$, and then drawing an edge uniformly at random (see \cite{Boyd2006a} for more details). Under this assumption, we can state a method detailed in Algorithm~\ref{alg:gossip_dual_averaging_async}.

The main difficulty in the asynchronous setting is that each node $i$ has to use a time estimate $m_i$ instead of the global clock reference (that is no longer available in such a context). Even if the time estimate is unbiased, its variance puts an additional error term in the convergence rate. However, for an iteration $T$ \emph{large enough}, one can bound these estimates as stated bellow.
\begin{lemma}
  \label{lma:bound_time_estimates}
  There exists $T_1 > 0$ such that for any $ t \geq T_1$, any $k \in [n]$ and any $q > 0$,
    \[
t^{-}:= t - t^{\frac{1}{2} + q}  \leq m_k(t) \leq t +  t^{\frac{1}{2} + q}=:t^{+} \text{ a.s.}
  \]
\end{lemma}
\begin{proof}
  Let $k \in [n]$. For $t \geq 1$, let us define $\delta_k(t)$ such that $\delta_k(t) = 1$ if $k$ is picked at iteration $t$ and $\delta_k(t) = 0$ otherwise. Then one has $m_k(t) = (1 / p_k)\sum_{s = 1}^t \delta_k(t)$. Since $(\delta_k(t))_{t \geq 1}$ is a Bernoulli process of parameter $1 / p_k$, by the law of iterative logarithms~\cite{dudley2010distances}, \citep[Lemma 3]{nedic2011asynchronous} one has with probability 1 and for any $q > 0$
  \[
    \lim_{t \rightarrow +\infty}\frac{|m_k(t) - t|}{t^{\frac{1}{2} + q}}  = 0,
  \]
  and the result holds.
\end{proof}

\begin{theorem}
 Let $\mathcal{G}$ be a connected and non bipartite graph.
  Let $(\gamma(t))_{t \geq 1}$ be defined as $\gamma(t)= c/ t^{1/2+\alpha}$ for some constant $c>0$ and $\alpha \in(0,1/2)$. For $i \in [n]$, let $(d_i(t))_{t \geq 1}$, $(g_i(t))_{t \geq 1}$, $(\epsilon_i(t))_{t \geq 1} $, $(z_i(t))_{t \geq 1} $ and $(\theta_i(t))_{t \geq 1} $ be generated as stated previously. For $\theta^* \in \argmin_{\theta' \in \bbR^d} R_n(\theta')$, $i \in [n]$ and $T > 0$, one has for some $C$:
\begin{align}
        R_n(\avtheta_i(T)) - R_n(\theta^*) \leq C  \max(T^{-\alpha/2},T^{\alpha-1/2}) + \frac{1}{T} \sum_{t = 1}^T \mathbb{E}_t[\overline{\epsilon}^n(t)^{\top} \omega(t)] \, .
\end{align}
\end{theorem}

\begin{proof}
  In the asynchronous case, for $i \in [n]$ and $t \geq 1$, one has
  \[
    \avtheta_i(T) = \frac{1}{m_i(T)} \sum_{t = 1}^T \frac{\delta_i(t)}{p_i} \theta_i(t).
  \]
  Then, using the convexity of $R_n$, one has:
  \begin{align}\label{ineq:cvx_asynch}
        \bbE_T[R_n(\avtheta_i(T)] - R_n(\theta^*) \leq \bbE_T\left[ \frac{1}{m_i(T)} \sum_{t = 1}^T \frac{\delta_i(t)}{p_i} R_n(\theta_i(t)) \right] - R_n(\theta^*).
    \end{align}
  By Lemma~\ref{lma:bound_time_estimates}, one has for $q > 0$
  \[
    \bbE_T[R_n(\avtheta_i(T)] - R_n(\theta^*) \leq \frac{1}{T^{-}} \sum_{t = 1}^T \bbE_T\left[\frac{\delta_i(t)}{p_i} R_n(\theta_i(t)) \right] - R_n(\theta^*).
  \]
  Similarly to the synchronous case, one can write
  \begin{align*}
    \bbE_T\left[ \frac{\delta_i(t)}{p_i} \avf(\theta_i(t)) \right]
    &= \sum_{j = 1}^n\frac{1}{n} \bbE_T\left[ \frac{\delta_i(t)}{p_i} f_j(\theta_i(t))\right] \\
    &= \frac{1}{n} \sum_{j = 1}^n \bbE_T\left[ \frac{\delta_i(t)}{p_i} (f_j(\theta_i(t)) - f_j(\theta_j(t)) \right] + \frac{1}{n} \sum_{j = 1}^n \bbE_T\left[ \frac{\delta_i(t)}{p_i} f_j(\theta_j(t)) \right].
  \end{align*}
  In order to use the gradient inequality, we need to introduce $\delta_j(t) f_j(\theta_j(t))$ instead of $\delta_i(t) f_j(\theta_j(t))$. For $j \in [n]$, one has:
  \[
    \frac{1}{T^{-}} \sum_{t = 1}^T \bbE_T\left[ \frac{\delta_i(t)}{p_i} f_j(\theta_j(t)) \right] = \frac{1}{T^{-}} \sum_{t = 1}^T \bbE_T\left[ \left( \frac{\delta_i(t)}{p_i} - \frac{\delta_j(t)}{p_j} \right) f_j(\theta_j(t)) \right] + \frac{1}{T^{-}} \sum_{t = 1}^T \bbE_T\left[ \frac{\delta_j(t)}{p_j} f_j(\theta_j(t)) \right].
  \]
  Let $N_j = \sum_{t = 1}^T \delta_j(t)$ and let $1 \leq t_1 < \ldots < t_{N_j} \leq T$ be such that $\delta_j(t_k) = 1$ for $k \in [N_j]$. One can write
  \begin{align}\label{ineq:horrible}
    \frac{1}{T^{-}} \sum_{t = 1}^T \bbE_T\left[ \left( \frac{\delta_i(t)}{p_i} - \frac{\delta_j(t)}{p_j} \right) f_j(\theta_j(t)) \right]
    =& \frac{1}{T^{-}} \bbE_T\left[ \sum_{k = 1}^{N_j - 1} \left( \left( \sum_{t = t_k}^{t_{k + 1} - 1} \frac{\delta_i(t)}{p_i}\right) - \frac{1}{p_j} \right) f_j(\theta_j(t_k)) \right]\nonumber\\
    &+ \frac{1}{T^{-}} \bbE_T \left[ \left( \sum_{t = 0}^{t_1} \frac{\delta_i(t)}{p_i}\right) f_j(\theta_j(0)) \right]\nonumber\\
    &+ \frac{1}{T^{-}} \bbE_T \left[ \left( \left( \sum_{t = t_{N_j}}^{T} \frac{\delta_i(t)}{p_i}\right) - \frac{1}{p_j} \right) f_j(\theta_j(t_{N_j})) \right]\nonumber\\
    \leq& +\frac{1}{T^{-}} \bbE_T\left[ \sum_{k = 1}^{N_j - 1} \left( \left( \sum_{t = t_k}^{t_{k + 1} - 1} \frac{\delta_i(t)}{p_i}\right) - \frac{1}{p_j} \right) f_j(\theta_j(t_k)) \right]\nonumber\\
    &+\frac{f_j(0)}{p_i p_j T^{-}} + \frac{L_f^2  \bbE_T[ \gamma(t_{N_j}-1)]}{p_i p_j} \, .
  \end{align}
  We need to study the behavior of $\delta_i$ and $\delta_j$ in the first term of the right hand side. One can check that
  \begin{align*}
    \bbE_T\left[ \sum_{k = 1}^{N_j - 1} \left( \left( \sum_{t = t_k}^{t_{k + 1} - 1} \frac{\delta_i(t)}{p_i}\right) - \frac{1}{p_j} \right) f_j(\theta_j(t_k)) \right]
    &= \bbE_T\left[ \sum_{k = 1}^{N_j - 1} \left( \bbE\left[ \sum_{t = t_k}^{t_{k + 1} - 1} \frac{\delta_i(t)}{p_i} \Bigg| t_k, t_{k + 1}\right] - \frac{1}{p_j} \right) f_j(\theta_j(t_k)) \right].
  \end{align*}
  $\delta_i(t)$ will not have the same dependency in $t_k$ whether $i$ and $j$ are connected or not. Let us first assume that $(i, j) \in E$. Then,
  \[
    \bbE[\delta_i(t_k) | t_k] = \bbE[\delta_i(t) | \delta_j(t) = 1 ] = \frac{1}{d_j}.
  \]
  Also, for $t_k < t < t_{k + 1}$, we get:
  \[
    \bbE[\delta_i(t) | t_k] = \bbE[\delta_i(t) | \delta_j(t) = 0 ] = \frac{p_i - 2 / |E|}{1 - p_j}.
  \]
  Finally, if $(i, j) \in E$, we obtain
  \[
    \bbE\left[ \sum_{t = t_k}^{t_{k + 1} - 1} \frac{\delta_i(t)}{p_i} \Bigg| t_k, t_{k + 1}\right] = \left(\frac{1}{d_j} + (t_{k + 1} - t_k - 1) \frac{p_i - 2 / |E|}{1 - p_j} \right)\frac{1}{p_i}.
  \]
  Before using this relation in the full expectation, let us denote that since $t_{k+1} - t_k$ is independent from $t_k$, one can write
  \[
    \bbE\left[ \left(\frac{1}{d_j} + (t_{k + 1} - t_k - 1) \frac{p_i - 2 / |E|}{1 - p_j} \right)\frac{1}{p_i} \Bigg | t_k \right]
    = \left(\frac{1}{d_j} + \left( \frac{1 - p_j}{p_j} \right) \frac{p_i - 2 / |E|}{1 - p_j} \right)\frac{1}{p_i}
    = \frac{1}{p_j}.
  \]
  We can now use this relation in the full expectation
  \begin{align}\label{eq:proba}
    \bbE_T\left[ \left( \frac{\delta_i(t)}{p_i} - \frac{\delta_j(t)}{p_j} \right) f_j(\theta_j(t)) \right]
    &= \bbE_T\left[ \sum_{k = 1}^{N_j - 1} \left( \bbE \left[ \bbE\left[ \sum_{t = t_k}^{t_{k + 1} - 1} \frac{\delta_i(t)}{p_i} \Bigg| t_{k + 1} - t_k \right] \Bigg | t_k \right] - \frac{1}{p_j} \right) f_j(\theta_j(t_k)) \right]=0 .
  \end{align}
  Similarly if $(i, j) \not\in E$, one has
  \[
    \bbE[\delta_i(t_k) | t_k] = \bbE[\delta_i(t) | \delta_j(t) = 1 ] = 0,
  \]
  and for $t_k < t < t_{k + 1}$,
  \[
    \bbE[\delta_i(t) | t_k] = \bbE[\delta_i(t) | \delta_j(t) = 0 ] = \frac{p_i}{1 - p_j},
  \]
  so the result of Equation \eqref{eq:proba} holds in this case. We have just shown that for every $j \in [n]$, we can use $\delta_j(t) f_j(\theta_j(t)) / p_j$ instead of $\delta_i(t) f_j(\theta_j(t)) / p_i$ . Combining \eqref{ineq:cvx_asynch} and \eqref{ineq:horrible} yields:
  \begin{align}
    \bbE_T[R_n(\avtheta_i(T))] - R_n(\theta^*)
     \leq & \frac{1}{n T^{-}} \sum_{t = 2}^T \sum_{j = 1}^n \bbE_T\left[ \frac{\delta_i(t)}{p_i} (f_j(\theta_i(t)) - f_j(\theta_j(t)) \right] \label{ineq:fifj_term}\\
    &+ \frac{1}{n T^{-}} \sum_{t = 2}^T \sum_{j = 1}^n \bbE_T\left[ \frac{\delta_j(t)}{p_j} \left( f_j(\theta_j(t)) - f_j(\theta^*) \right) \right]\\
    &+ \frac{1}{T^{-}} \sum_{t = 2}^T \bbE_T\left[ \frac{\delta_i(t)}{p_i} (\psi(\theta_i(t)) - \psi(\theta^*)) \right]\label{ineq:psi_term}\\
       &+\frac{f_j(0)}{p_i p_j T^{-}} + \frac{L_f^2  \bbE_T[ \gamma(t_{N_j}-1)]}{p_i p_j} \, .
  \end{align}
  Let us focus on the second term of the right hand side. For $t \geq 2$, one can write
  \begin{align}
    \frac{1}{n} \sum_{j = 1}^n \bbE_T\left[ \frac{\delta_j(t)}{p_j} \left( f_j(\theta_j(t)) - f_j(\theta^*) \right) \right]
    \leq& \frac{1}{n} \sum_{j = 1}^n \bbE_T\left[ \frac{\delta_j(t)}{p_j} g_j(t)^{\top} (\theta_j(t) - \theta^*)  \right] \nonumber\\
    =& \frac{1}{n} \sum_{j = 1}^n \bbE_T\left[ \frac{\delta_j(t)}{p_j} g_j(t)^{\top} (\theta_j(t) - \omega(t))  \right] \label{ineq:gj}\\
    &+ \frac{1}{n} \sum_{j = 1}^n \bbE_T\left[ \frac{\delta_j(t)}{p_j} g_j(t)^{\top} (\omega(t) - \theta^*)  \right]\label{ineq:gstar}
  \end{align}
$\bullet$ Here we control the term from \eqref{ineq:gstar} using $\omega(t):=\smoothop_{m_i(t)}(\avz(t))$
  \begin{align*}
    \frac{1}{n} \sum_{j = 1}^n \bbE_T\left[ \frac{\delta_j(t)}{p_j} g_j(t)^{\top} (\omega(t) - \theta^*)  \right]
    &= \bbE_T\left[ \left( \frac{1}{n} \sum_{j = 1}^n \frac{\delta_j(t)}{p_j} g_j(t)\right)^{\top} (\omega(t) - \theta^*)  \right]\\
    &= \bbE_T\left[ \avg(t)^{\top} (\omega(t) - \theta^*)  \right],
  \end{align*}
  and the reasoning of the synchronous case can be applied to obtain
  \begin{align}
    \frac{1}{n T^{-}} \sum_{t = 2}^T \sum_{j = 1}^n \bbE_T\left[ \frac{\delta_j(t)}{p_j} g_j(t)^{\top} (\omega(t) - \theta^*)  \right]
    \leq & \frac{L_f^2}{2T^{-} } \sum_{t = 2}^T \gamma(t - 1) + \frac{\|\theta^*\|^2}{2 \gamma(T)}\nonumber\\
    &+ \frac{1}{T} \sum_{t = 2}^T \mathbb{E}_t [\overline{\epsilon}^n(t)^{\top}\omega(t)]\nonumber\\
    &+ \frac{1}{T^{-}} \sum_{t = 2}^T (\psi(\theta^*) - \bbE_T[\psi(\omega(t))]). \label{ineq:32}
  \end{align}

Let us regroup the term from \eqref{ineq:32} and \eqref{ineq:psi_term} together:

\begin{align}
    \frac{1}{T^{-}} \sum_{t = 2}^T \bbE_T\left[ \frac{\delta_i(t)}{p_i} (\psi(\theta_i(t)) - \psi(\theta^*)) \right]
    +
    \frac{1}{T^{-}} \sum_{t = 2}^T (\psi(\theta^*) - \bbE_T[\psi(\omega(t))])
    = & \frac{1}{T^{-}} \sum_{t = 2}^T \bbE_T\left[ \frac{\delta_i(t)}{p_i} \psi(\theta_i(t)) - \psi(\omega(t)) \right]\nonumber\\
    = &
\frac{1}{T^{-}} \sum_{t = 2}^T \bbE_T\left[ \frac{\delta_i(t)}{p_i} (\psi(\theta_i(t)) - \psi(\omega(t)) ) \right]\nonumber\\
&+ \frac{1}{T^{-}} \sum_{t = 2}^T \bbE_T \left[ (\frac{\delta_i(t)}{p_i}-1) \psi(\omega(t))\right]\nonumber\\
=& \frac{1}{T^{-}} \sum_{t = 2}^T \bbE_T\left[ \frac{\delta_i(t)}{p_i} (\psi(\theta_i(t)) - \psi(\omega(t)) ) \right] \, ,
\end{align}
where we have used for the last term the same arguments as in \eqref{eq:proba} to state $ \frac{1}{T^{-}} \sum_{t = 2}^T \bbE_T \left[ (\frac{\delta_i(t)}{p_i}-1) \psi(\omega(t))\right]=0$. Then, one can use the fact that $\smoothop_t$ is $\gamma(t)$-Lipschitz to write:
\begin{equation*}
\frac{1}{p_i T^{-}} \sum_{t = 2}^T \bbE_T\left[ 2 L_f \gamma(m_i(t-1))  \|\avz(t) - z_i(t)\| + \frac{\gamma(m_i(t-1))\|\avz(t) - z_i(t)\|^2}{2(m_i(t-1))} \right].
\end{equation*}
Provided that $\gamma(t)\leq \frac{C}{\sqrt{t}}$ for some constant $C$, then using Lemma \ref{lma:bound_time_estimates} we can bound this term by $\frac{C'}{\sqrt{T}}$.

$\bullet$ Now we control the term in \eqref{ineq:gj} as follows:

\begin{align}
    \frac{1}{n} \sum_{j = 1}^n \bbE_T\left[ \frac{\delta_j(t)}{p_j} g_j(t)^{\top} (\theta_j(t) - \omega(t))  \right] \label{ineq:thetaj_omega}
    \leq &  \frac{L_f}{n p_j} \sum_{j =
     1}^n \bbE_T\left[  \|\theta_j(t) - \omega(t) \| \right] \\
    \leq &
\frac{L_f}{n p_j} \sum_{j = 1}^n \bbE_T\left[  \|\theta_j(t) - \tilde \theta_j(t) \| + \|\tilde \theta_j(t)- \omega(t) \| \right]\\
\leq & \frac{L_f}{n p_j} \sum_{j = 1}^n \bbE_T\left[ \gamma(m_j(t-1)) \|z_j(t) -  \avz(t) \| + \|\tilde \theta_j(t)- \omega(t) \| \right].
\end{align}
where $\tilde \theta_j(t)=\smoothop_{m_j(t-1)}(-\avz(t))$.
We can apply Lemma \ref{lma:thisistheend} with the choice $\theta_1=\tilde\theta_j(t)$, $\theta_2=\omega(t)$, $t_1=m_j(t)$, $t_2=m_i(t)$ and $z=\avz(t)$.

\begin{align}\label{ineq:theta_1_2_partial}
  \|\omega(t)-\tilde\theta_j(t)\|
  \leq & \|\avz(t)\|
  \Bigg( |\gamma(m_i(t)) - \gamma(m_j(t))| + \nonumber\\
&
  \left(\frac{3}{2} +\max(\frac{\gamma(m_j(t))}{\gamma(m_i(t))},\frac{\gamma(m_i(t))}{\gamma(m_j(t))}) \right) \left(\frac{1}{m_j(t)}+\frac{1}{m_i(t)}\right)
   |m_j(t) \gamma(m_j(t))-m_i(t) \gamma(m_i(t))| \Bigg) \,.
\end{align}
We use Lemma \ref{lma:bound_time_estimates} with the choice $q=\alpha/2$, so we can bound for $t$ large enough the former expression by a term of order $\|\avz (t) \| |\gamma(m_i(t)) - \gamma(m_j(t))|$.  Note also that $\|\avz(t)\|\leq L_f \max_{k=1,\ldots,n}m_k(t)$, so for $t$ large enough we obtain:
\begin{align}
  \|\omega(t)-\tilde\theta_j(t)\| \leq L_F t^{+} |\gamma(t^-) - \gamma(t^+)| \,.
\end{align}
With the additional constraint that $\gamma(t)=C t^{-1/2-\alpha}$,
$\|\omega(t)-\tilde\theta_j(t)\|$ is bounded by $C't^{-\alpha/2}$ for $t$ large enough, and so is
$\frac{1}{n} \sum_{j = 1}^n \bbE_T\left[ \frac{\delta_j(t)}{p_j} g_j(t)^{\top} (\theta_j(t) - \omega(t))  \right]$.

$\bullet$ To control the term in \eqref{ineq:fifj_term} we use that $f_j$ is $L_f$-Lipschitz
\begin{align}
    |f_j(\theta_i(t)) - f_j(\theta_j(t)|
    \leq &L_f\|\theta_i(t)-\theta_j(t)\|\\
    \leq &L_f(\|\theta_i(t)-\omega(t)\|+\|\omega(t)-\theta_j(t)\|).
\end{align}
and we use now the same control as for \eqref{ineq:thetaj_omega}, hence the result.

\end{proof}

\begin{lemma}\label{lma:thisistheend}
  Let $\gamma : \bbR_+ \rightarrow \bbR_+$ be a non-increasing positive function and let $z \in \bbR^d$. For any $t_1, t_2 >  0$, one has
\begin{align}\label{ineq:theta_1_2_final}
  \|\theta_2-\theta_1\|
  \leq & \|z\|
  \left( |\gamma(t_2) - \gamma(t_1)|+
  \left(\frac{3}{2} +\max(\frac{\gamma(t_1)}{\gamma(t_2)},\frac{\gamma(t_2)}{\gamma(t_1)}) \right) \left(\frac{1}{t_1}+\frac{1}{t_2}\right)
   |t_1 \gamma(t_1)-t_2 \gamma(t_2)| \right),
\end{align}
where
\begin{align*}
    \theta_1 = \smoothop_{t_1}(z) :=  &\argmax_{\theta \in \bbR^d}
    \left\{
    z^{\top} \theta - \frac{\| \theta \|^2}{2 \gamma(t_1)}  - t_1 \psi(\theta) \right\}\\
    \theta_2 = \smoothop_{t_2}(z) :=  &\argmax_{\theta \in \bbR^d}
    \left\{
    z^{\top} \theta - \frac{\| \theta \|^2}{2 \gamma(t_2)}  - t_2 \psi(\theta) \right\}.
\end{align*}
\end{lemma}

\begin{proof}
Using the optimality property of the minimizers, for any $s_1\in \partial \psi(\theta_1)$ (resp. $s_2\in \partial \psi(\theta_2)$):
\begin{align*}
    (\gamma(t_1) z- t_1 \gamma(t_1)s_1  - \theta_1)^\top (\theta_2-\theta_1) \leq 0
    \\
    (\gamma(t_2) z- t_2 \gamma(t_2)s_2  - \theta_2)^\top (\theta_1-\theta_2) \leq 0
\end{align*}
Re-arranging the terms, and using properties of sub-gradients yields:
\begin{align}\label{ineq:theta_sqr_1_2}
  \|\theta_2-\theta_1\|^2
  \leq& (\gamma(t_2) - \gamma(t_1))  z^\top (\theta_2-\theta_1) + (t_1 \gamma(t_1)s_1-t_2 \gamma(t_2)s_2)^\top(\theta_2-\theta_1)\\
  \leq&(\gamma(t_2) - \gamma(t_1))  z^\top (\theta_2-\theta_1) + (t_1 \gamma(t_1)-t_2 \gamma(t_2))(\psi(\theta_2)-\psi(\theta_1))
\end{align}

Also, using the definition of $\theta_1$ and $\theta_2$, one has:
\begin{align}\label{ineq:theta_1_2}
    |\psi(\theta_1)-\psi(\theta_1)|\leq  \|z\| \|\theta_1-\theta_2\| \left(\frac{3}{2} +\max(\frac{\gamma(t_1)}{\gamma(t_2)},\frac{\gamma(t_2)}{\gamma(t_1)}) \right) \left(\frac{1}{t_1}+\frac{1}{t_2}\right).
\end{align}

With relation \eqref{ineq:theta_sqr_1_2} and \eqref{ineq:theta_1_2} we bound the distance between $\theta_1$ and $\theta_2$ as follows:

\begin{align}
  \|\theta_2-\theta_1\|
  \leq & \|z\|
  \left( |\gamma(t_2) - \gamma(t_1)|+
  \left(\frac{3}{2} +\max(\frac{\gamma(t_1)}{\gamma(t_2)},\frac{\gamma(t_2)}{\gamma(t_1)}) \right) \left(\frac{1}{t_1}+\frac{1}{t_2}\right)
   |t_1 \gamma(t_1)-t_2 \gamma(t_2)| \right)
\end{align}

\end{proof}



\section{Extension to Multiple Points per Node}
\label{app:multi}

For ease of presentation, we have assumed throughout the paper that each node $i$ holds a single data point $x_i$. In this section, we discuss simple extensions of our results to the case where each node holds the same number of points $k\geq 2$.
First, it is easy to see that our results still hold if nodes swap their entire set of $k$ points (essentially viewing the set of $k$ points as a single one). However, depending on the network bandwidth, this solution may be undesirable.

We thus propose another strategy where only two data points are exchanged at each iteration, as in the algorithms proposed in the main text. The idea is to view each ``physical'' node $i\in V$ as a set of $k$ ``virtual'' nodes, each holding a single observation. These $k$ nodes are all connected to each other as well as to the neighbors of $i$ in the initial graph $\mathcal{G}$ and their virtual nodes. Formally, this new graph $\mathcal{G}^\otimes=(V^\otimes,E^\otimes)$ is given by $\mathcal{G}\times \mathcal{K}_k$, the tensor product between $\mathcal{G}$ and the $k$-node complete graph $\mathcal{K}_k$. It is easy to see that $|V^\otimes|=kn$ and $|E^\otimes|=k^2|E|$. We can then run our algorithms on $\mathcal{G}^\otimes$ (each physical node $i\in V$ simulating the behavior of its corresponding $k$ virtual nodes) and the convergence results hold, replacing $1 - \lambda_2^{\mathcal{G}}$ by $1 - \lambda_2^{\mathcal{G}^\otimes}$ in the bounds. The following result gives the relationship between these two quantities.

\begin{proposition}
\label{prop:gap}
Let $\mathcal{G}$ be a connected, non-bipartite and non-complete graph with $n$ nodes. Let $k\geq 2$ and let $\mathcal{G}^\otimes$ be the tensor product graph of $\mathcal{G}$ and $\mathcal{K}_k$. Let $1 - \lambda_2^{\mathcal{G}}=\beta_{n-1}^{\mathcal{G}}/|E|$ and $1 - \lambda_2^{\mathcal{G}^\otimes}=\beta_{kn-1}^{\mathcal{G}^\otimes}/|E^\otimes|$, where $\beta_{n-1}^{\mathcal{G}}$ and $\beta_{kn-1}^{\mathcal{G}^\otimes}$ are the second smallest eigenvalues of $L(\mathcal{G})$ and $L(\mathcal{G}^\otimes)$ respectively. We have that
$$1 - \lambda_2^{\mathcal{G}^\otimes} = \frac{1}{k}\left(1 - \lambda_2^{\mathcal{G}}\right).$$
\end{proposition}
\begin{proof}
Let $A\in\{0,1\}^{n\times n}$ and $A^\otimes\in\{0,1\}^{nk\times nk}$ be the adjacency matrices of $\mathcal{G}$ and $\mathcal{G}^\otimes$ respectively. Similarly, let $D\in\mathbb{N}^{n\times n}$ and $D^\otimes\in\mathbb{N}^{nk\times nk}$ be the diagonal degree matrices of $\mathcal{G}$ and $\mathcal{G}^\otimes$ respectively, \textit{i.e.}, $D_{ii}=\sum_{j=1}^n A_{ij}$ and $D^\otimes_{ii}=\sum_{j=1}^{nk} A^\otimes_{ij}$. Denoting the Kronecker product by $\otimes$, we can write:
\begin{eqnarray*}
A^\otimes &=& \1_k\1_k^T \otimes A,\\
D^\otimes &=& kI_k \otimes D.
\end{eqnarray*}
Recall that $L(\mathcal{G})=D-A$ and $L(\mathcal{G}^\otimes) = D^\otimes - A^\otimes$.

Let $(v,\beta^{\mathcal{G}^\otimes})\in\mathbb{R}^{nk}\times \mathbb{R}$ be an eigenpair of $L(\mathcal{G}^\otimes)$, \textit{i.e.}, $(D^\otimes - A^\otimes)v = \beta^{\mathcal{G}^\otimes} v$ and $v\neq\mathbf{0}_{nk}$. Let us write $v=[v_1 \dots v_k]^{\top}$ where $v_1,\dots,v_k\in\mathbb{R}^n$. Exploiting the structure of $A^\otimes$ and $D^\otimes$, we have:
\begin{equation}
\label{eq:propeq1}
kDv_i - \sum_{j=1}^k A v_j = \beta^{\mathcal{G}^\otimes} v_i,\quad\forall i\in\{1,\dots,k\}.
\end{equation}
Summing up \eqref{eq:propeq1} over all $i\in\{1,\dots,k\}$ gives
$$D\sum_{i=1}^k v_i - A\sum_{i=1}^k v_i = \frac{\beta^{\mathcal{G}^\otimes}}{k}\sum_{i=1}^k v_i,$$
which shows that if $(v,\beta^{\mathcal{G}^\otimes})$ is an eigenpair of $L(\mathcal{G}^\otimes)$ with $\sum_{i=1}^k v_i \neq \mathbf{0}_n$, then $(\sum_{i=1}^k v_i,\beta^{\mathcal{G}^\otimes}/k)$ is an eigenpair of $L(\mathcal{G})$.
In the case where $\sum_{i=1}^k v_i = \mathbf{0}_n$, then there exists an index $j\in\{1,\dots,k\}$ such that $v_j=-\sum_{i\neq j}v_j\neq \mathbf{0}_n$. Hence \eqref{eq:propeq1} gives
$$Dv_j = \frac{\beta^{\mathcal{G}^\otimes}}{k}v_j,$$
which shows that $(v_j,\beta^{\mathcal{G}^\otimes}/k)$ is an eigenpair of $L(\mathcal{G})$. Observe that $\beta^{\mathcal{G}^\otimes}=kd_i$ for some $i\in\{1,\dots,n\}$.

We have thus shown that any eigenvalue $\beta^{\mathcal{G}^\otimes}$ of $L(\mathcal{G}^\otimes)$ is either of the form $\beta^{\mathcal{G}^\otimes}=k\beta^{\mathcal{G}}$, where $\beta^{\mathcal{G}}$ is an eigenvalue of $L(\mathcal{G})$, or of the form $\beta^{\mathcal{G}^\otimes}=kd_i$ for some $i\in\{1,\dots,n\}$.

Since $L(\mathcal{G}^\otimes)$ is a Laplacian matrix, its smallest eigenvalue is 0. Let $\beta^{\mathcal{G}^\otimes}_{nk-1}$ be the second smallest eigenvalue of $L(\mathcal{G}^\otimes)$. Note that $\mathcal{G}^\otimes$ is not a complete graph since $\mathcal{G}$ is not complete. Therefore, $\beta^{\mathcal{G}^\otimes}_{nk-1}$ is bounded above by the vertex connectivity of ${G}^\otimes$ \citep{Fiedler1973a}, which is itself trivially bounded above by the minimum degree $d^\otimes_{min} = \min_{i=1}^{kn} D^\otimes_{ii}$ of ${G}^\otimes$. This implies that $\beta^{\mathcal{G}^\otimes}_{nk-1}=k\beta_{n-1}^{\mathcal{G}}$, and hence
$$1 - \lambda_2^{\mathcal{G}^\otimes}=\frac{\beta_{kn-1}^{\mathcal{G}^\otimes}}{|E^\otimes|} = \frac{k\beta_{n-1}^{\mathcal{G}}}{k^2|E|} = \frac{1}{k}(1 - \lambda_2^{\mathcal{G}}).$$
\end{proof}

Proposition~\ref{prop:gap} shows that the network-dependent term in our convergence bounds is only affected by a factor $k$. Furthermore, note that iterations involving two virtual nodes corresponding to the same physical node will not require actual network communication, which somewhat attenuates this effect in practice.


\section{Additional experiments}
\label{sec:addit-exper}

\begin{figure}[t]
  \centering
  \includegraphics[width=.99\textwidth]{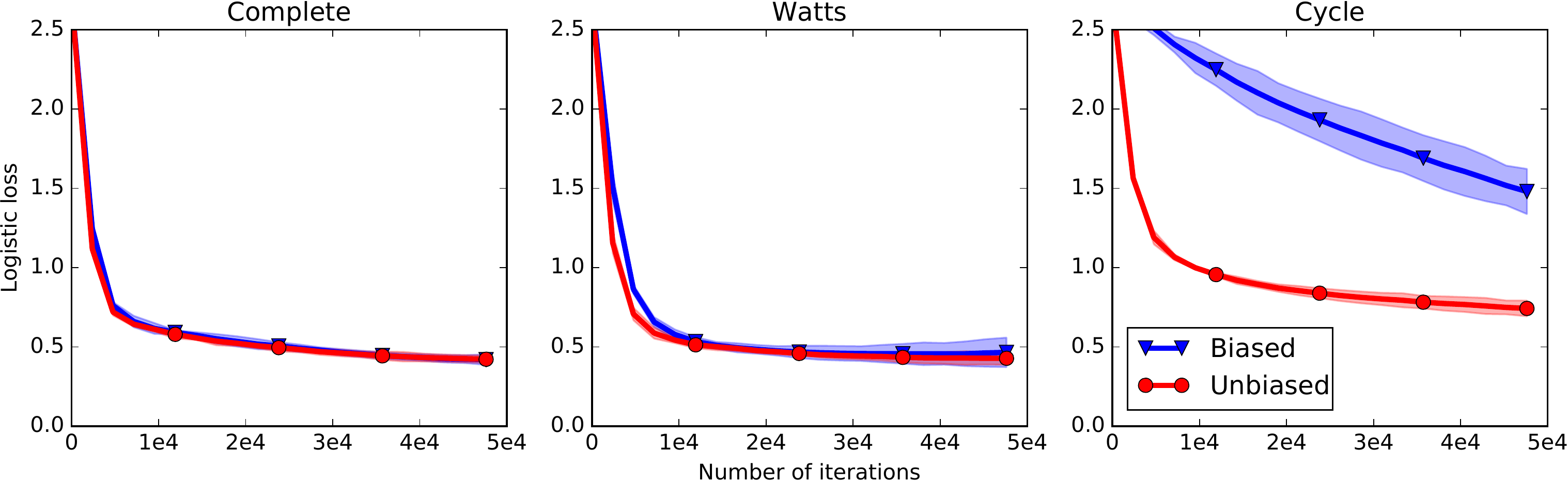}
  \caption{Metric learning: comparison between our algorithm and an unbiased version}
  \label{fig:async_ml_baseline-vs-gossip}
\end{figure}

\begin{figure*}[t]
  \centering
\subfigure[Evolution of the objective function and its standard deviation (asynchronous case).]{\hspace*{.5cm}  \includegraphics[height=6cm]{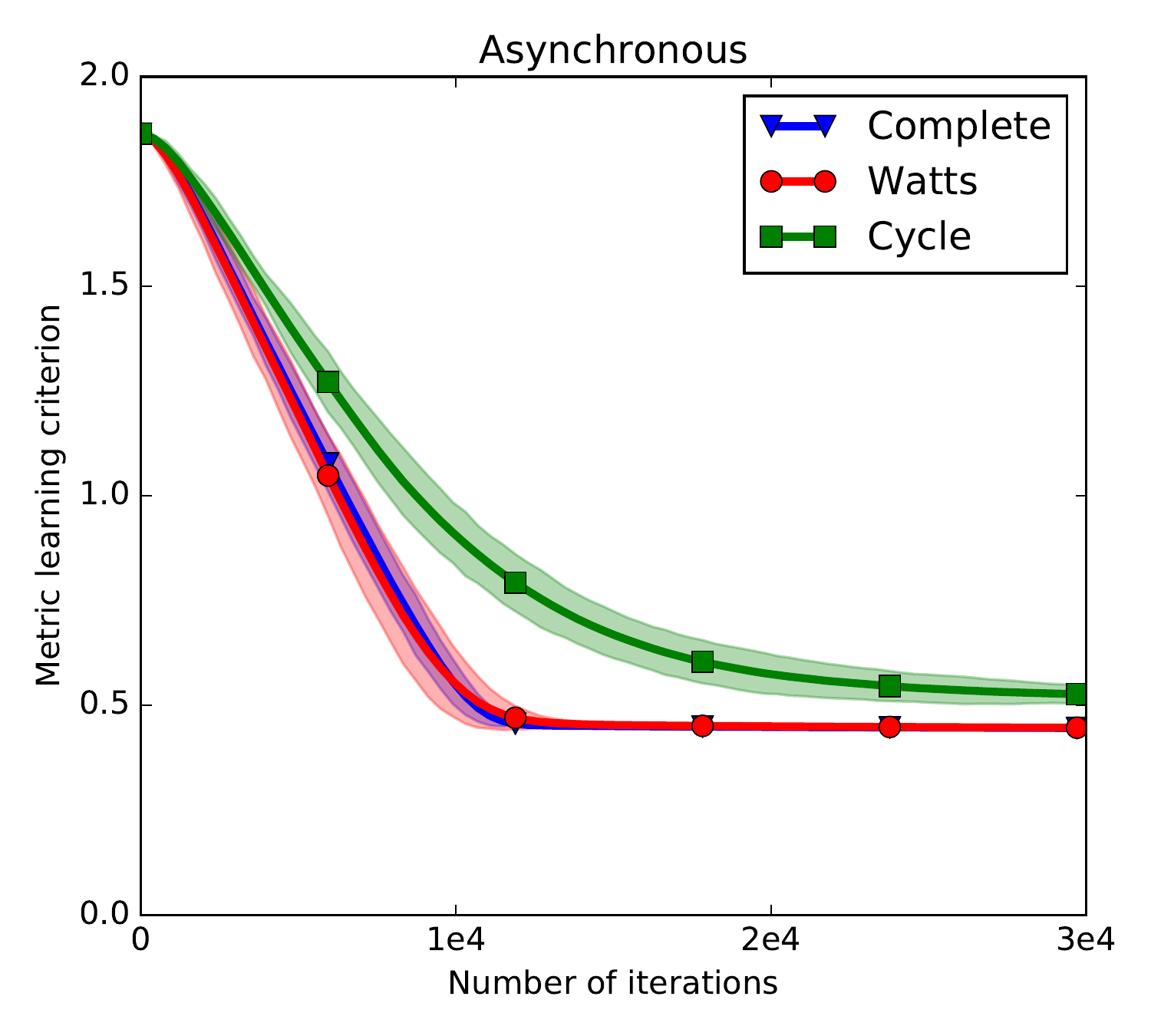}\label{fig:async_ml_cancer_standard}}\hspace*{.5cm}
  \subfigure[Evolution of the bias term.]{\hspace*{.5cm}  \includegraphics[height=6cm]{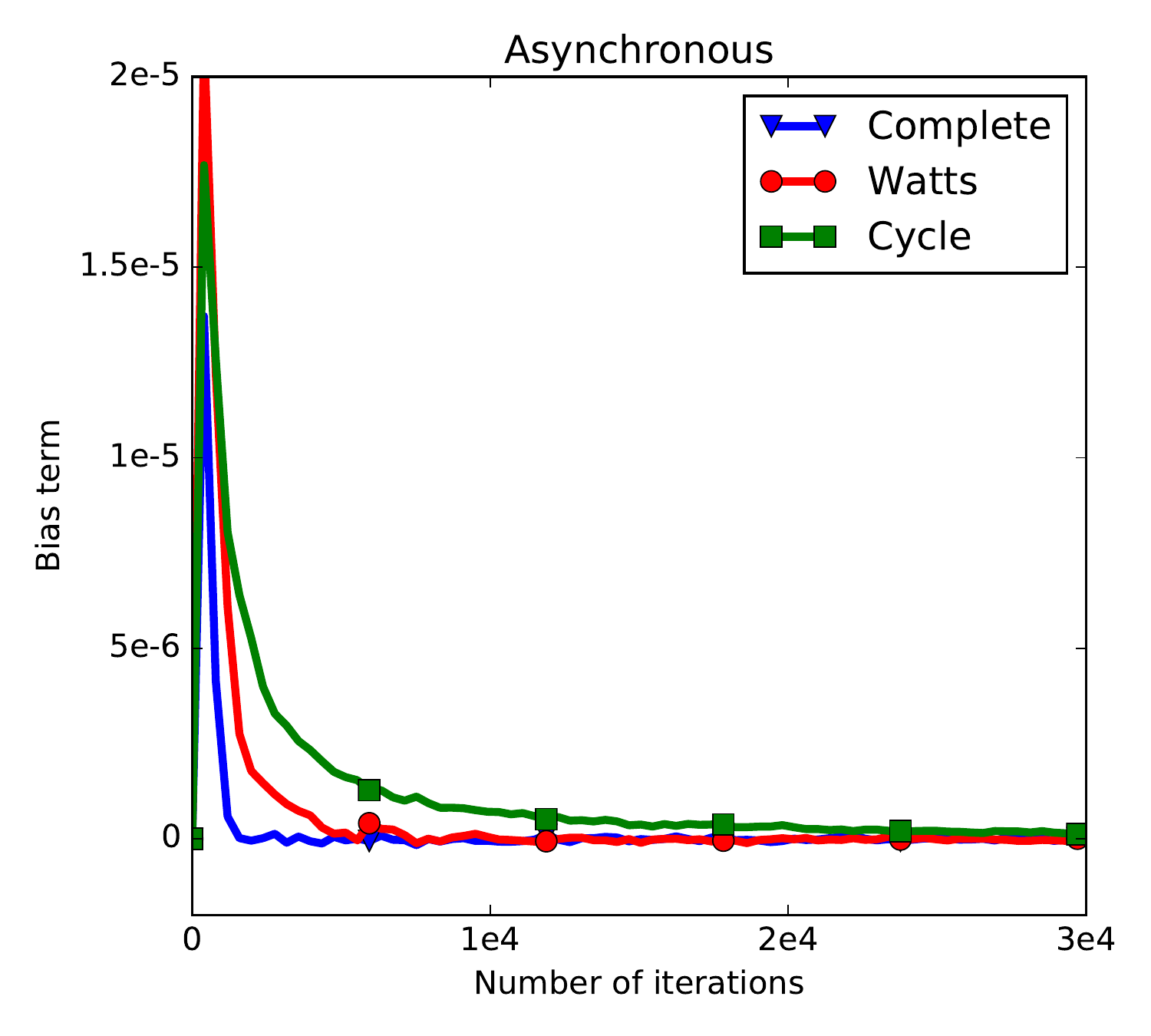}\hspace*{.5cm}\label{fig:async_ml_cancer_bias}}
  \caption{Metric learning experiments on a real dataset.}
\end{figure*}

In this section, we present additional results of decentralized metric learning. First, we discuss the comparison to the unbiased basline for metric learning on the synthetic dataset introduced in Section~\ref{sec:experiments}. Then, we analyze numerical experiments of decentralized metric learning on the Breast Cancer Wisconsin dataset\footnote{\url{https://archive.ics.uci.edu/ml/datasets/Breast+Cancer+Wisconsin+(Original)}}.

\paragraph{Synthetic Dataset}
In Section~\ref{sec:experiments}, we discussed the results of decentralized metric learning over a synthetic dataset of $n=1,000$ points generated from a mixture of $10$ Gaussians in $\mathbb{R}^{40}$ such that all gaussian means are contained in a 5d subspace.

We compare the logistic loss associated to our algorithm's iterates to the loss associated to the following baseline: instead of adding $\nabla f(\theta_i(t); x_i, y_i(t))$ to its dual variable $z_i(t)$, a node $i \in [n]$ receives a vector drawn uniformly at random from the set $\{ \nabla f(\theta_i(t); x_i, x_1), \ldots, \nabla f(\theta_i(t); x_i, x_n) \}$. The bias introduced by the random walk procedure is already shown to be very small in comparison to the objective function on Figure~\ref{fig:async_ml_bias}. Here, Figure~\ref{fig:async_ml_baseline-vs-gossip} evidences the fact that this small bias has close to no influence on the optimization process for well-connected networks.

\paragraph{Breast Cancer Wisconsin Dataset}

We now focus on decentralized metric learning on the Breast Cancer Wisconsin Dataset already used in Section~\ref{sec:experiments} for AUC maximization. This dataset contains $n=699$ observations of dimension $11$. Figure~\ref{fig:async_ml_cancer_standard} shows the evolution of the metric learning criterion with the number of iterations, averaged over $50$ runs. As in previous experiments, there is almost no difference between the convergence rate of the Watts-Strogatz network and the complete network. Moreover, the bias term is again largely negligible when compared to the metric learning criterion, as shown on Figure~\ref{fig:async_ml_cancer_bias}.

\clearpage

\section*{Acknowledgments}

This work was partially supported by the chair ``Machine Learning for Big Data'' of T\'el\'ecom ParisTech and by a grant from CPER Nord-Pas de Calais/FEDER DATA Advanced data science and technologies 2015-2020.

\bibliographystyle{icml2016}
\bibliography{doc}
\end{document}